\documentclass[letterpaper]{article}
\usepackage{graphicx} \usepackage{tikz-cd}
\usepackage{amssymb}
\usepackage{dsfont}
\usepackage{mleftright}
\usepackage{ntheorem}
\usepackage{mathtools}
\usepackage[colorlinks=true, allcolors=blue]{hyperref}
\usepackage{algorithm, algpseudocodex}
\usepackage{cleveref}
\usepackage{natbib}
\usepackage{multirow}
\usepackage{amsmath,amsfonts,graphicx}
\usepackage{xfrac}
\usepackage{natbib}
\usepackage{dsfont}
\usepackage{footmisc}
\usepackage{makecell}
\usepackage{multirow} 
\usepackage{enumitem}

\newtheorem{theorem}{Theorem}[section]
\newtheorem{lemma}[theorem]{Lemma}

\newtheorem{corollary}[theorem]{Corollary}
\newtheorem{definition}[theorem]{Definition}

\newtheorem*{remark}{Remark}
\newenvironment{proof}{{\bf Proof:}}{\hfill\rule{2mm}{2mm}}
\newenvironment{proofsketch}{{\bf Proof sketch:}}{\hfill\rule{2mm}{2mm}}

\newcommand\E[1]{\mathbb{E}\squary{#1}}
\newcommand{\bbE}{\mathbb{E}}

\newcommand{\A}{\mathcal{A}}

\newcommand{\rr}{\mathbb{R}}

\newcommand{\calF}{\mathcal{F}}
\newcommand{\bbP}{\mathbb{P}}

\newcommand\inprod[1]{\left\langle #1 \right\rangle}
\newcommand{\series}[1]{\inprod{#1}}
\newcommand\indicator[1]{\mathds{1}\left[ {#1} \right]}
\newcommand{\curly}[1]{ {\left\{ #1 \right\}}}
\newcommand{\roundy}[1]{ {\left( #1 \right)}}
\newcommand{\squary}[1]{ {\left[ #1 \right]}}
\newcommand{\abs}[1]{ {\left | #1 \right |}}
\newcommand{\telta}{{\tilde{\Delta}}}
\newcommand{\tmu}{\tilde{\mu}}

\newcommand{\logp}[1]{\log\roundy{#1}}
\newcommand{\lcb}[2]{\text{lcb}_{#1}\roundy{#2}}
\newcommand{\ucb}[2]{\text{ucb}_{#1}\roundy{#2}}
\newcommand{\sucb}[1]{\textnormal{ucb}^*\roundy{#1}}
\newcommand{\olcb}[2]{\overline{\text{lcb}}_{#1}\roundy{#2}}
\newcommand{\oucb}[2]{\overline{\text{ucb}}_{#1}\roundy{#2}}
\newcommand{\width}[2]{\text{width}_{#1}\roundy{#2}}
\newcommand{\owidth}[1]{\overline{\text{width}}\roundy{#1}}
\newcommand{\omu}[2]{\overline{\mu}_{#1}\roundy{#2}}
\newcommand{\Nmax}[2]{{N_{#1}^{#2}}}

\newcommand{\ceil}[1]{\left \lceil #1 \right\rceil}

\makeatletter
\newcounter{algorithmicH}
\let\oldalgorithmic\algorithmic
\renewcommand{\algorithmic}{%
  \stepcounter{algorithmicH}
  \oldalgorithmic}
\renewcommand{\theHALG@line}{ALG@line.\thealgorithmicH.\arabic{ALG@line}}
\makeatother

\newcommand{\mail}[1]{\href{mailto:#1}{\color{blue} #1}}
\newcommand{\blackfootnote}[1]{%
  \begingroup
  \hypersetup{allcolors=black}%
  \footnote{#1}%
  \endgroup
}

\author{Ofir Schlisselberg\blackfootnote{Tel Aviv University, \mail{ofirs4@mail.tau.ac.il}} \and
Tal Lancewicki\blackfootnote{Tel Aviv University, \mail{lancewicki@mail.tau.ac.il}} \and
Peter Auer\blackfootnote{Technical University of Leoben, \mail{auer@unileoben.ac.at}} \and
Yishay Mansour\blackfootnote{Tel Aviv University and Google Research, \mail{mansour.yishay@gmail.com}}}

\title{Improved Best-of-Both-Worlds Regret for Bandits with Delayed Feedback}

\begin{document}

\maketitle

\begin{abstract}
We study the multi-armed bandit problem with adversarially chosen delays in the Best-of-Both-Worlds (BoBW) framework, which aims to achieve near-optimal performance in both stochastic and adversarial environments. While prior work has made progress toward this goal, existing algorithms suffer from significant gaps to the known lower bounds, especially in the stochastic settings. Our main contribution is a new algorithm that, up to logarithmic factors, matches the known lower bounds in each setting individually. 

In the adversarial case, our algorithm achieves regret of $\widetilde{O}(\sqrt{KT} + \sqrt{D})$, which is optimal up to logarithmic terms, where $T$ is the number of rounds, $K$ is the number of arms, and $D$ is the cumulative delay. In the stochastic case, we provide a regret bound which scale as $\sum_{i:\Delta_i>0}\roundy{\logp T/\Delta_i} + \frac{1}{K}\sum \Delta_i \sigma_{max}$, where $\Delta_i$ is the sub-optimality gap of arm $i$ and $\sigma_{\max}$ is the maximum number of missing observations. 

To the best of our knowledge, this is the first \textit{BoBW} algorithm to simultaneously match the lower bounds in both stochastic and adversarial regimes in delayed environment. Moreover, even beyond the BoBW setting, our stochastic regret bound is the first to match the known lower bound under adversarial delays, improving the second term over the best known result by a factor of $K$.
        \end{abstract}

\section{Introduction}
Delayed feedback presents a significant challenge that sequential decision-making algorithms encounter in many real-world applications. Notably, delays are often an inherent part of environments involving sequential decision-making, such as in healthcare, finance, and recommendation systems. As a central challenge in Online Learning, delays have been extensively explored in various contexts within Multi-armed Bandits (MAB), both in stochastic settings, where losses are generated i.i.d. from a fixed underlying distribution  \citep{joulani2013online,pike2018bandits,zhou2019learning,vernade2017stochastic,vernade2020linear,cesa2018nonstochastic,manegueu2020stochastic,lancewicki2021stochastic,wu2022thompson,howson2022delayed,tang2024stochastic,Schlisselberg_Cohen_Lancewicki_Mansour_2025} and adversarial settings, where the losses are chosen arbitrarily by an adversary \citep{quanrud2015online,cesa2016delay,thune2019nonstochastic,bistritz2019online,pmlr-v108-zimmert20a,ito2020delay,gyorgy2020adapting,van2021nonstochastic,van2023unified}. 

Roughly speaking, under stochastic losses, delays contribute an additive regret term that does not scale with the time horizon (but with the number of missing observations), whereas under the adversarial losses, delays introduce an additive term that does scale with the horizon. More specifically, for an arbitrary sequence of delays, the best-known regret under stochastic losses is $\sum_{\Delta_i > 0} \frac{\logp T}{\Delta_i} + K\sigma_{max}$ (\citet{joulani2013online}) where $T$ is the number of rounds, $K$ is the number of arms, $\Delta_i$ is the sub-optimality gap of arm $i$ and $\sigma_{max}$ is the maximal number of missing observations. Under adversarial losses, the optimal bound is of the order $\sqrt{TK} + \sqrt{D}$ (\citet{thune2019nonstochastic,bistritz2019online}), where $D$ is the sum of the delays.

While the regret bounds of delayed Multi-armed Bandit under stochastic losses and under adversarial losses are well understood separately, the following question remains open:

\begin{center} 
    \textbf{\textit{Is there a single algorithm that, without knowing the nature of the losses a-priori in delayed environment, can achieve the optimal regret bounds in both regimes simultaneously?}} 
\end{center}

Such an algorithm is often referred to as a \textit{best-of-both-worlds} algorithm. \citet{masoudian2022best, masoudian2024best} have made significant progress toward answering this question. Their regret bound is $O(\sqrt{TK} + \sqrt{D} + K\sigma_{max} + \Phi^*)$ in the adversarial regime and $O\roundy{\sum_{i:\Delta_i>0} \roundy{\frac{\logp T}{\Delta_i } + \frac{\sigma_{max}}{\logp K \Delta_i}} + \Phi^*}$ in the stochastic regime, where $\Phi^*=\min\{d_{max}K^{2/3},\sqrt{DK^{2/3}}\}$. However, these bounds are still not optimal.

\begin{table}[t]
    \caption{Comparison of regret bounds (up to constants) to the previous state-of-the-art regret both under stochastic and adversarial losses under adversarial delays.}
    \begin{center}
        \begin{tabular}[c]{l l l}
            \hline
            Algorithm & Regime & Regret
            \\
            \hline
            \citet{joulani2013online} & stochastic & $\sum_{\Delta_i > 0} (\frac{\logp T}{\Delta_i} + \sigma_{max} \Delta_i$)
            \\
            \hline
            \makecell[l]{\citet{thune2019nonstochastic}\\ \citet{bistritz2019online}\\ \citet{pmlr-v108-zimmert20a}\footnotemark{}}& adversarial & $\sqrt{TK} + \sqrt{D}$
            \\
            \hline
            \multirow{3}{*}{\citet{masoudian2024best}\footnotemark[\value{footnote}]} & stochastic & $\sum_{i:\Delta_i > 0} (\frac{\logp T}{\Delta_i} + \frac{\sigma_{max}}{\Delta_i \logp K})+\Phi^*$
            \\
                                & adversarial & $\sqrt{TK} + \sqrt{D}+\Phi^*+K\sigma_{\max}$
            \\
            &&$\Phi^* = \min\curly{d_{max}K^{2/3},\sqrt{DK^{2/3}}}$
            \\
            \hline
            \multirow{2}{*}{Our paper\footnotemark[\value{footnote}]} & stochastic & $\sum_{i:\Delta_i > 0} (\frac{\logp T}{\Delta_i} + \sigma_{max}\frac{\Delta_i}{K})$
            \\
                                & adversarial & $\sqrt{T K \logp 
                                T} + \sqrt{D}$
                                \\
            \hline
            \underline{Lower Bound}\\ 
            \citet{lancewicki2021stochastic} (constant delay)& stochastic & $\sum_{i:\Delta_i > 0} (\frac{\logp T}{\Delta_i} + \sigma_{max}\frac{\Delta_i}{K})$ 
            \\
                            \citet{masoudian2022best}\footnotemark[\value{footnote}]    & adversarial & $\sqrt{T K} + \sqrt{D}$ 
                                \\
        \end{tabular}
        \label{table: comparison}
    \end{center}

\end{table}
\footnotetext{In these papers the $\sqrt{D}$ is actually $\min_{S\in[T]}\curly{\abs{S}+\sqrt{D_{\bar{S}}}}$, where $D_{\bar{S}}$ is the total delay of the steps not in $S$. We wrote the worst-case for the simplicity of the table.
}

\textbf{Our contributions.} In this work, we affirmatively answer the above question and present a new best-of-both-worlds algorithm for Multi-armed Bandits (MAB) with delayed feedback that simultaneously achieves the near-optimal regret bounds under both stochastic and adversarial losses. Specifically: 
\begin{itemize}[itemsep=0pt, parsep=0pt, topsep=0pt, partopsep=0pt,left=5pt]
    \item In the adversarial regime our algorithm guarantees optimal $\tilde O(\sqrt{TK} +\sqrt{D})$ regret.
    \item In the stochastic regime our algorithm guarantees optimal $O(\sum_{i\ne i^\star} (\frac{\logp T}{\Delta_i } + \frac{1}{K}{\sigma_{max}\Delta_i}))$ regret.
\end{itemize}
In the adversarial regime, compared to \citet{masoudian2024best} we have an extra logarithmic factor in the $\sqrt{TK}$ term, which is independent of the delay. However, we eliminate the additive $\Phi^*$ in their bound, which is significant when $d_{max}$ is very large; even a single large delay causes the regret to scale as $\sqrt{DK^{2/3}}$ rather than our $\sqrt{D}$ delay term, which is tight to the lower bound of \citet{masoudian2022best}.

Even more significantly, in the stochastic regime, our bound improves the $O(\sum_{i} \frac{\sigma_{max}}{\Delta_i \logp K} + \Phi^*)$ term from the bound of \citep{masoudian2024best} to $O(\frac{1}{K} \sum_{i} \sigma_{max} \Delta_i)$.  That is, for each term in the summation, we achieve an improvement by a factor of $\frac{K}{\Delta_i^2\logp K}$. This is a significant improvement. For example, consider the simple case of fixed delay $d$, which implies $\sigma_{max}=d$, and constant number of actions. 
For any sub-optimality gaps our regret is at most 
$\sqrt{T} + d$ 
while there is a setting where the regret of \citep{masoudian2024best} is at least
$ \sqrt{dT}$. Moreover, if the maximum delay is large, $\Phi^*$ can be as large as $\sqrt{D}$, offering no improvement over the additive delay term in the adversarial setting.

Our bound in the stochastic regime represents an improvement even compared to state-of-the-art results for algorithms specifically designed for the stochastic case. Specifically, \citet{joulani2013online} provides the best-known result for stochastic losses with adversarial delays where their bound includes an additive term of $\sum_{i\ne i^*} \sigma_{max} \Delta_i$, which we improve by a factor of $\Theta(K)$. While \citet{lancewicki2021stochastic} reduce this dependence on $K$,  their result applies only to the case of stochastic delays. Moreover, their regret bound scales with the maximal sub-optimality gap, rather than the average. For example, in the simple case of a fixed delay $d$, their additive term is $d \max_i \Delta_i$, whereas ours is $\frac{d}{K} \sum_i \Delta_i$, offering a strictly better dependence on the problem parameters in many scenarios.

\subsection{Additional Related work}
\textbf{Delayed MAB with stochastic losses.} The problem was first addressed  by \citet{dudik2011efficient}, who analyzed the case of constant delays and established a regret bound with linear dependence on the delay. This line of work was extended by \citet{joulani2013online}, who allowed the delays to change through time. Subsequent work introduced several important refinements: \citet{zhou2019learning} distinguished between arm-dependent and arm-independent delays; \citet{pike2018bandits} introduced an aggregated rewards model where only the sum of rewards that arrive at the same round is observed; and \citet{lancewicki2021stochastic} studied delays in the contexts of reward-dependent or reward-independent delays. More recently, \citet{tang2024stochastic}, \citet{Schlisselberg_Cohen_Lancewicki_Mansour_2025} and \citet{zhang2025contextual} studied settings in which the delay is equal to the payoff.

\textbf{Delayed MAB with adversarial losses.}
Delayed feedback have also been explored in adversarial settings, where both rewards and delays can be chosen adversarially. \citet{quanrud2015online} studied this problem in the full-information setting. The bandit setting was first addressed by \citet{cesa2019delay}, who analyzed the case of constant delay. This line of work was extended by \citet{thune2019nonstochastic} and \citet{bistritz2019online}, who considered general adversarial delays under the assumption that the delay is known at the time the arm is pulled. Subsequently, \citet{gyorgy2020adapting} and \citet{pmlr-v108-zimmert20a} removed this assumption and analyzed the case where the delay is unknown at the time of the action. Finally, \citet{van2021nonstochastic} extended the setting to allow for arm-dependent delays.

\textbf{``Best of Both Worlds" without delays.}
The "Best of Both Worlds" framework in multi-armed bandits was introduced by \citet{bubeck2012towards}, who proposed an algorithm that initially follows a stochastic-style strategy but switches to a standard adversarial algorithm upon detecting signs of adversarial losses. This adaptive approach was further developed by \citet{auer2016algorithmnearlyoptimalpseudoregret}. An alternative perspective is to start with an adversarial-style algorithm and prove that it achieves instance-dependent regret bounds in stochastic settings as well. In this direction, \citet{pmlr-v32-seldinb14} and \citet{pmlr-v65-seldin17a} adapted the EXP3 algorithm to perform well in both regimes, while \citet{pmlr-v89-zimmert19a,pmlr-v195-dann23a,pmlr-v247-ito24a} extended this idea to Follow-The-Regularized-Leader (FTRL), achieving optimal performance across both adversarial and stochastic settings.

\section{Settings}\label{sec:settings}

We study the Multi-armed Bandit (MAB) problem with delayed feedback, summarized in Protocol~\ref{protocol}. 
In each round $t=1,2,\ldots,T$, an agent chooses an arm $a_{t}\in\left[K\right]$
and suffers loss $\ell_{t}(a_t)$, where $\ell_t(\cdot) \in [0,1]^K$ can be either stochastic or adversarial. Under the stochastic regime for each $i \in [K]$, $\{\ell_t(i)\}_{t=1}^T \overset{i.i.d} {\sim} \mathcal{D}_i$ where $\mathcal{D}_i$ is some distribution with expectation $\mu_i$. Under the adversarial regime the loss sequence $\{\ell_t\}_{t=1}^T$ are chosen arbitrarily by an oblivious adversary. Unlike the standard MAB setting, the agent does not immediately observe $\ell_{t}(a_t)$ at the end of round $t$; rather,
only after $d_{t}$ rounds (namely, at the end of round $t+d_{t}$) the tuple $(t, \ell_t(a_t))$ is received as feedback. 
The delays $\{d_t\}_{t=1}^T$ are chosen by an oblivious adversary. 

 \begin{algorithm}[t]
    \floatname{algorithm}{Protocol}
    \caption{\label{protocol} Delayed MAB}
    \begin{algorithmic}[1]
        \For{$t\in \left[T\right]$}
        
            \State Agent picks an action $a_t \in [K]$.
            \State Agent incurs loss $\ell_t(a_t)$ and observes feedback $\left\{(\ell_s(a_s),d_s) : t = s + d_s \right\}$.
        \EndFor
    \end{algorithmic}
    \end{algorithm}

The performance of the agent is measured as usual by the 
the difference between the algorithm's cumulative expected loss and the best possible total expected reward of any fixed arm:
\begin{align*}
    \mathcal{R}_{T} 
    &= \mathbb{E}\left[\sum_{t=1}^{T}\ell_{t}(a_t)\right] - 
    \min_{i}\mathbb{E}\left[\sum_{t=1}^{T} \ell_{t}(i)\right]  .
\end{align*}
In the stochastic case the regret can also be written as,
\begin{align*}
    \mathcal{R}_{T} 
    & = \mathbb{E}\left[\sum_{t=1}^{T}\mu_{a_{t}}\right] - T\mu_{i^*}  = \mathbb{E}\left[\sum_{t=1}^{T}\Delta_{a_t}\right]
    ,
\end{align*}
 where $i^*$ denotes the optimal arm and $\Delta_{i} = \mu_{i} - \mu_{i^*}$ for all $i \in [K]$.

\textbf{Additional notation.} We denote the total delay by $D = \sum_{t=1}^T d_t$ and the maximal delay by $d_{max} = \max_{t\in[T]} d_t$. The amount of missing feedback at time $t$ is defined by $\sigma(t) = |\{\tau \mid \tau \leq t, \tau + d_\tau > t\}|$ and the maximum over $\sigma(t)$ is denoted by $\sigma_{max} = \max_{t\in[T]} \sigma(t)$. The rounds observed before and available at round $t$ are denoted by $B(t) = \curly{s : s + d_s < t}$.

\textbf{Notation for the algorithms.}
Let $S$ denotes a sequence of rounds that the algorithm process. $S_{:n}$ is the first $n$ elements in $S$ and $S_{:-n}$ is $S$ except for the last $n$ elements.  $n_i(S)$ is the number of pulls of arm $i$ in the rounds of $S$, $ \hat{\mu}_i(S) = \frac{1}{n_i(S)}\sum_{s\in S : a_s = i}l_i(s)$ is the empirical mean over $S$ and ${\text{width}}_i(S) = 
\min\big\{1, \sqrt{\frac{2\logp{T}}{n_i(S)}}\big\}$
is a confidence width. $\text{ucb}_i(S) = \min\curly{\hat{\mu}_i + \width{i}{S},\ucb{i}{S_{:-1}}}$ and $\text{lcb}_i(S) = \max\curly{\hat{\mu}_i - \width{i}{S},\lcb{i}{S_{:-1}}}$ are upper and lower confidence bounds with respect to the empirical average.
The algorithm also maintains confidence bounds around an average importance sampling estimator. Let $\overline{L}_i(S) = \sum_{s\in S}\frac{\indicator{a_s=i}\ell_i(s)}{p_i(s)}$ be the sum of the estimators over rounds in $S$, and $\overline{\mu}_i(S) = \frac{1}{\abs{S}}\overline{L}_i(S)$ be the average. We also define $\overline{\text{width}}(S) = \min\big\{1, \sqrt{\frac{2K\logp{T}}{\abs{S}}}\big\}$, $\overline{\text{lcb}}_i(S) = \max\curly{\omu{i}{S} - \owidth{S},\olcb{i}{S_{:-1}}}$ and $\overline{\text{ucb}}_i(S) = \min\curly{\omu{i}{S} + \owidth{S},\oucb{i}{S_{:-1}}}$. Finally, we define $\sucb{S} = \min_i\curly{\ucb{i}{S},\oucb{i}{S}}$.

\section{Algorithm}
Our algorithm, sketched in \Cref{alg:main sketch} and formally described in \Cref{alg:main}, builds on the \texttt{SAPO} algorithm of \citet{auer2016algorithmnearlyoptimalpseudoregret}. The main idea is to integrate an external algorithm for adversarial settings, \texttt{ALG}. Our algorithm initially follows a stochastic-like strategy while monitoring whether the environment exhibits stochastic behavior. If this assumption is violated, it switches to \texttt{ALG}.  

At its core, the algorithm is based on a successive elimination (SE) framework \cite{even2006action}, maintaining a set of active arms played with equal probability. It tracks a confidence bound, \texttt{width}, which defines upper and lower estimates for each arm's mean. When an arm is found to be non-optimal, it is eliminated. However, unlike standard SE methods, the algorithm continues to play eliminated arms but with reduced probability. This accounts for the possibility that losses are adversarial—an arm that appears suboptimal at one point may later turn out to be optimal. 
To verify the stochastic nature of arms, the algorithm employs the \texttt{BSC} procedure to assess the nature of active arms, and a more advanced procedure \texttt{EAP} for assessing and determining the sampling probability of non-active arms.

\begin{algorithm}[h]
\caption{Sketch of Delayed SAPO Algorithm}
\label{alg:main sketch}
\begin{algorithmic}[1]
\Require Number of arms $K$, number of rounds $T \geq K$, Algorithm \texttt{ALG}.
\State Initialize active arms $\mathcal{A} = \{1, \dots, K\}$, $S = \series{}$
\For{$t = 1, 2, \dots, T$}
    \For{$s\in B(t) \setminus S$} \Comment{Iterating newly received feedback}
        \State $S = S + \series{s}$
        \If{not $\texttt{BSC}(S)$} \Comment{Non-stochastic behavior on active arms (Procedure \ref{alg:bsc main})}
                \State Switch to \texttt{ALG}.
        \EndIf
        \State $\mathcal{A} = \mathcal{A} \setminus \{i \in \mathcal{A}: \hat{\mu}_{i}(S) - 9\width{i}{S} > \sucb{S}\}$ \Comment{Elimination} \label{alg line: elimination}
    \EndFor
    
    \For{$i \in \bar{\A}$}
        \State $p_i(t), err = \texttt{EAP}(i)$ \Comment{Get the reduced probability for the non-active arm (Procedure \ref{alg:eap main})}
        \If{$err$} \Comment{Non-stochastic behavior on nonactive arms}
            \State Switch to \texttt{ALG}.  
        \EndIf
    \EndFor
    \State $\forall i\in \A\quad p_i(t) = \left(1 - \sum_{j \in \bar{\A}(t)} p_j(t)\right)/|\mathcal{A}(t)|$ \Comment{Equal probability for active arms}
    \State Sample $a_t \sim p(t)$, observe feedback and update variables
\EndFor
\end{algorithmic}
\end{algorithm}

\textbf{Basic Stochastic Checks (\texttt{BSC}) Subroutine.} This procedure performs two checks. The first ensures that an unbiased estimate of the mean of each arm remains within its confidence interval, expanded by an additional radius. In the stochastic regime, using standard concentration bounds we have that with high probability,
\begin{align*}
    \olcb{i}{S} \le \mu_i \le \omu{i}{S} + \owidth{S};\qquad
     \omu{i}{S} - \owidth{S} \le \mu_i \le \oucb{i}{S}.
\end{align*}
Thus, in \cref{alg:first switch main} of Procedure~\ref{alg:bsc main},
we check that the above conditions are met. 

The second check in \texttt{BSC} constructs a lower bound on the regret and verifies that it is indeed smaller than the expected regret in the stochastic regime, which can be shown to be $\tilde O(\sqrt{TK} + \sigma_{max})$ under stochastic losses. To define this lower bound, we use the fact that, with high probability, $\mu^* \leq \sucb{S}$. Thus, $\sum_{s'\in S}\roundy{l_{a_{s'}}(s') - \text{ucb}^*(S)}$ is a lower bound on the regret, which forms the condition in \cref{alg:second switch main} of the procedure.

\begin{algorithm}[h] \label{alg:bsc main}
\floatname{algorithm}{Procedure}
\caption{Basic Stochastic Checks (BSC) Subroutine}
\label{alg:bsc main}
\begin{algorithmic}[1]
\Require Series of processed pulls $S$
\If{\(\exists i \in \mathcal{A} : \omu{i}{S} \not\in [\olcb{i}{S} - \owidth{S}, \oucb{i}{S} + \owidth{S}]\)}\label{alg:first switch main}
                \State \Return False 
        \EndIf
        \If{
            $\sum_{s'\in S}\roundy{l_{a_{s'}}(s') - \text{ucb}^*(S)} > 272\sqrt{KT\logp{T}} + 10\sigma_{max}(t)\logp{K}$
        }\label{alg:second switch main}
            \State \Return False 
        \EndIf
    \Return True
\end{algorithmic}
\end{algorithm}

\textbf{Eliminated Arms Processing (EAP) Subroutine.} Since we do not know in advance whether we are in the stochastic or adversarial regime, we cannot completely eliminate an arm — if we did, the adversary could assign losses of $0$ after elimination of an arm, and we would never detect this. Therefore, we maintain a positive sampling probability even for eliminated actions. EAP maintains these probabilities for eliminated arms and checks whether the estimated loss is significantly smaller than the empirical mean at elimination. Intuitively, if we are in the stochastic regime, we want the probability of playing an eliminated arm to decrease over time. Conversely, if we suspect the loss after elimination is significantly smaller than the empirical mean at the elimination time, we increase that arm's probability. If there is sufficient confidence that the arm does not behave stochastically, we switch to the adversarial algorithm.

In more detail, the probability of playing an eliminated arm $i$ is updated in discrete phases. Let $\tilde{S}_i$ be the set of processed rounds at the time of elimination of arm $i$. We denote $\telta_i = 8\text{width}_{i}({\tilde S_i})$, i.e the width at elimination time. As we'll later see in the analysis, $\telta_i$ is indeed a good estimate of the sub-optimality gap of arm $i$ in the stochastic case (see \Cref{lem:delta bound}). Each phase $r$ has a maximum length $N_i^r = \Theta(1/(p_i^r \telta_i^2))$, where $p_i^r$ is the sample probability of arm $i$ in its $r$th phase and $p_i^1 = \frac{1}{2K} + \frac{n_i(\tilde S_i)}{2T}$.\footnote{We note that the initial probability assigned in the first phase differs from that in \citet{auer2016algorithmnearlyoptimalpseudoregret}, and is crucial for obtaining adversarial regret bound that scales with $\sqrt{KT}$ instead of $K\sqrt{T}$ achieved in \cite{auer2016algorithmnearlyoptimalpseudoregret}.} This value is always $\Omega(1/K)$, but can be as high as a uniform probability over the active arms at the time of elimination. 
If we reach the maximum length $N_i^r$, then we have acquired additional $N_i^r p_i^r = \Theta(1/\telta_i^2)$ samples from arm $i$. In this case, we halve the sampling probability of arm $i$ and start a new phase with a doubled maximum length (\cref{alg: EAP phase ended}). During the phase, we monitor whether the average importance sampling estimate of the loss $\bar{\mu}(S_i^r)$ is smaller than $\tilde{\mu}_i =\hat{\mu}_i(\tilde{S}_i)$ by more than $\Theta(\telta_i N_i^r / |S_i^r|)$, where $S_i^r$ is the sequence of processed rounds in phase $r$. If this condition is met, referred to as a ``phase error", we terminate the phase but now double the sampling probability of arm $i$ and halve the maximum phase length accordingly (\cref{alg:error phase}). 

In the stochastic regime, phase errors occur with a constant probability, but the probability that they will happen $\Theta(\logp T)$ times is negligible. In such cases, we transition to the adversarial algorithm. 

During a phase with sampling probability $p$, we process only the observed rounds after elimination in which arm $i$ was played and the sampling probability was $p$. If a sample is observed with a different sampling probability $p'$, it is stored in a \textit{``probability bank"} which we denote by $B_i^{p'}$ and is processed only if a new phase is initiated with probability $p'$. The probability banks allow us to utilize most samples, even if they are observed after their respective phases end, and play an important role in removing a factor of the number of phases ($\Theta(\logp T)$) from the delay term in the regret.

\begin{algorithm}[t]
\floatname{algorithm}{Procedure}
\caption{Sketch of Eliminated Arms Processing (EAP) Subroutine }
\label{alg:eap main}
\begin{algorithmic}[1]
\State Note: Some of the variables here initialized in \Cref{alg:main sketch}
\Require Arm $i$

\State $p := p_i^{r_i}$, $\tilde{\mu}$ is the empirical average at the elimination time of $i$
\State Let $B_i^{p}$ be observed rounds after elimination in which $i$ was played and the sampling probability was $p$
\While{$B_i^{p} \setminus C_i^{p} \ne \emptyset$} 
\For{$s \in B_i^{p} \setminus C_i^{p}$}
    \State $C_i^{p} = C_i^{p} \cup \curly{s}$
    \State Let $S_i^{r_i}$ be the samples processes so far in phase $r_i$
    \If{$\abs{S_i^{r_i}}\tmu_i - \bar{L}\roundy{S_i^{r_i}} \ge \frac{1}{4}\telta_i \Nmax{i}{r_i}$} \Comment{phase error}
        \State $E_i = E_i + 1$, $\Nmax{i}{r_i+1} = \max\curly{\Nmax{i}{1},\frac{1}{2}\Nmax{i}{r_i}}$, $p_i^{r_i+1} = \min\curly{p_i^1,2p_i^{r_i}}$, $r_i=r_i+1$ \label{alg:error phase}
    \If{$E_i \ge 3\logp{T}$}
        \Return $0$, True \Comment{Switch to adversarial algorithm} \label{alg:third switch main}
    \EndIf
                    \State \textbf{break}
    \EndIf
    \If{$\abs{S_i^{r_i}} = \lfloor \Nmax{i}{r_i} \rfloor$} \Comment{phase ended} \label{alg: EAP phase ended}
        \State $\Nmax{i}{r_i+1} = 2\Nmax{i}{r_i}$, $p_i^{r_i+1} = \frac{1}{2}p_i^{r_i}$ $S_i^{r_i+1} = \series{}$, $r_i=r_i+1$ 
        \State \textbf{break}
        \label{alg:success phase} 
    \EndIf
\EndFor

\State $p := p_i^{r_i}$
\EndWhile

\State \textbf{return} $p_i^{r_i}$, False
\end{algorithmic}
\end{algorithm}

\section{Stochastic analysis}
\begin{theorem}
The regret in the stochastic settings is bounded by:
\begin{align*}
    \mathcal{R}_{sto} \le O\roundy{\sum_{i=1}^K \frac{\logp{T}}{\Delta_i} + \sigma_{max}\Delta_{avg}\logp{K}}
\end{align*}
\end{theorem}

The first term above is the optimal MAB regret under stochastic losses without delays. The second term is the additional regret due to delay and, in general, cannot be improved—except for the $\logp K$ factor, due to the lower bound for constant delays (see \Cref{table: comparison}). We note that with a more involved algorithm and analysis, we are able to eliminate the $\logp K$ factor and match this lower bound. For simplicity of presentation, the full details are deferred to \Cref{appendix:log K}. 
The dependence on $\sigma_{max}$ improves upon the BoBW result of \citet{masoudian2024best} in the stochastic regime by a factor of $\tilde{O}(K / \Delta_i^2)$ for each $i$. Moreover, it is tighter by a factor of $K$ compared to the best previous known algorithm that specifically designed for this regime (\citet{joulani2013online}).

\begin{proofsketch}
The total regret can be decomposed as,
\begin{align}
    \mathcal{R}_{sto} = \sum_{i\in [K]} m_i(T) \Delta_i = \sum_{i\in [K]} m_i(\tau_i) \Delta_i + \sum_{i\in [K]} (m_i(T) - m_i(\tau_i)) \Delta_i,
    \label{eq:sto regret decomposition}
\end{align}
where $m_i(t)$ is the number of pulls of arm $i$, up to time $t$ and $\tau_i$ is the elimination time of arm $i$. The first term above is the regret up to elimination and second term is the regret after elimination. (Recall that we need to keep sampling eliminated arms.)

\textbf{Regret up to elimination.} The regret before elimination analysis largely follows standard Stochastic Elimination (SE) with delayed feedback arguments. However, achieving dependence on $\Delta_{avg}$ rather than $\Delta_{max}$ necessitates a new algorithmic component and technical argument.
We start by further decomposing the regret up to elimination:
\begin{align*}
    \sum_{i\in [K]} m_i(\tau_i) \Delta_i = \sum_{i\in [K]}  n_i(\tau_i)\Delta_i + \sum_{i\in [K]} (m_i(\tau_i) - n_i(\tau_i))\Delta_i,
\end{align*}
where $n_i(t)$ is the number of \textit{observed} samples from arm $i$.
Similar to standard non-delayed SE analysis, we can show that with high probability, each suboptimal arm is eliminated whenever $\Theta\roundy{\frac{\logp T}{\Delta_i^2}}$ samples from arm $i$ have been observed. Thus, $n_i(\tau_i)\Delta_i = \Theta\roundy{\frac{\logp T}{\Delta_i}}$. For the second term above, recall that the number of missing feedback is bounded by $\sigma_{max}$; but only a fraction of the missing feedback is from arm $i$. Loosely speaking, if $p_{i}^{max} = \max_{t\leq \tau_i} p_i(t)$ is the maximal probability of sampling $i$ before elimination, then the number of missing feedback from arm $i$ at time $\tau_i$ is roughly bounded by $m_i(\tau_i) - n_i(\tau_i) \leq \sigma_{max} p_{i}^{max}$. Further note that if $\kappa_i$ is the number of active arms at the time of elimination then $p_{i}^{max} \leq \frac{1}{\kappa_i}$. Overall, the total regret up to elimination is bounded by
\begin{align*}
    O\roundy{\sum_i \frac{\logp T}{\Delta_i} + \sum_i \frac{\sigma_{max}}{\kappa_i}\Delta_i}
\end{align*}
For the second term, each $\Delta_i$ can be trivially bounded by $\Delta_{max}$, and $\sum_i 1/\kappa_i \leq \sum_i 1/i \leq 1 + \logp K$, resulting in $\sum_i \frac{\sigma_{max}}{\kappa_i}\Delta_i \leq O(\sigma_{max}\Delta_{max} \logp K)$. In order to have dependency with respect $\Delta_{avg}$ instead of $\Delta_{max}$ a more detailed argument is required. Unlike regular SE algorithms, an arm isn't eliminated when the $\text{ucb}$ of some other arm is lower than its $\text{lcb}$. Instead, the algorithm eliminates when there are multiple widths between the two (see \cref{alg line: elimination} in \cref{alg:main sketch}). This stricter condition ensures that arms are roughly eliminated in decreasing order of $\Delta_i$. Specifically, we show the following lemma:

\begin{lemma}
    \label{lem:deltas are monotone main}
    If arm $i_1$ was eliminated before $i_2$ then,
    $
        \Delta_{i_2} \leq 20\Delta_{i_1}
    $.
\end{lemma}

For the first half of eliminated arms where $\kappa_i \geq K / 2$, the additive delay term is at most order of $\sum_{i:\kappa_i > K/2} \frac{\sigma_{max}}{\kappa_i}\Delta_i \leq \sigma_{max}\Delta_{avg}$. Using the above lemma we show that for second half of eliminated arms $\Delta_i = O(\Delta_{avg})$, yielding an additive delay term of at most $O\roundy{\sigma_{\max} \Delta_{avg} \logp K}$. Overall we get that the regret up to elimination is bounded by $\sum_{i\in [K]} m_i(\tau_i) \Delta_i \lesssim \sum_i \frac{\logp T}{\Delta_i} + \sigma_{\max} \Delta_{avg} \logp K$.

\textbf{Regret after elimination.} 
For the regret after elimination, we break the number of pulls of arm $i$ after elimination for pulls that where processed by algorithm and pulls that where not processed by the algorithm (either because the feedback had not returned or the samples remained in the probability bank):
\begin{align}\label{eq:inactive pulls main}
    m_i(T) - m_i(\tau_i) = \sum_{r=1}^{r_i}n_i(S_i^r)  + \sum_{j=0}^{\logp{T}} n_i(M_i^{p_i^12^{-j}}),
\end{align}
where $r_i$ is the total number of phases of arm $i$, $S_i^r$ are the samples processes at phase $r$ and $M_i^{p}$ denotes the post-elimination rounds where the probability of pulling arm $i$ was $p$, but these were not processed by the algorithm (either because the feedback was not observed or the rounds remained unprocessed in the probability bank).

Recall that the maximum length of phase $r$ is $N_i^r = \Theta (1/(p_i^r \telta_i^2))$. Additionally, the fact that arms are only eliminated when the empirical average exceeds $\text{ucb}^*$ by more than multiple \texttt{widths} allows us to show that $\telta_i \approx \Delta_i$ (see \Cref{lem:delta bound}). Using standard concentration bounds, $n_i(S_i^r) \approx N_i^r p_i^r \approx 1/\Delta_i^2$. To bound the number of phases, note that the maximum phase length can be either doubled or halved. The number of times it is halved in the stochastic regime is at most $3 \logp T$ with high probability (see \Cref{lem:error prob,lem:switch3}), where in case of a failure event, we switch to the adversarial algorithm. Since the number of times it is halved is bounded by $O(\logp T)$, we can also bound the number of times it is doubled before reaching the time horizon $T$. Formally, in \Cref{lem:max phases}, we bound the total number of phases by $7\logp T$. Therefore the first term in \Cref{eq:inactive pulls main} is bounded by $O(\frac{log T}{\Delta_i^2})$ and the regret from these rounds is $O(\frac{log T}{\Delta_i})$.

For the second term of \cref{eq:inactive pulls main}, note that the size of $M_i^{p_i^12^{-j}}$ is at most $\sigma_{max}$, but only a small fraction of those rounds belongs to arm $i$. Since the probability of pulling arm $i$ in these rounds was  $p_i^12^{-j}$ we have that $n_i(M_i^{p_i^12^{-j}}) \approx \sigma_{max}p_i^12^{-j}$. Summing over this geometric series gives us $\sum_{j=0}^{\logp{T}} n_i(M_i^{p_i^12^{-j}}) = O(\sigma_{max}p_i^1)$. Recall that $p_i^1 = \frac{1}{2K} + \frac{n_i(\tau_i)}{2T}$. Since the probability of pulling arm $i$ before elimination is at most $1/\kappa_i$, where $\kappa_i$ is the number of active arms at the time of elimination, ${n_i(\tau_i)}/{T} \leq {n_i(\tau_i)}/{\tau_i} \lesssim 1/\kappa_i$.  That is, $p_i^1 \leq O(1/\kappa_i)$. We get that the total regret after elimination from unprocessed pulls (multiplying the second term in \cref{eq:inactive pulls main} by $\Delta_i$ and summing over $i$)  is of order $\sum_i \frac{\sigma_{max}}{\kappa_i}\Delta_i$. Again, leveraging the fact that arms are eliminated roughly in decreasing order of their sub-optimality gaps, we can bound the last sum by $\sigma_{max} \Delta_{avg} \logp K$.
\end{proofsketch}

\begin{remark}
    As mentioned in the proof sketch, our algorithm adds additional \texttt{width} to the elimination inequality, which makes the eliminated arms to be in descending order of their sub-optimality gap. We stress that this is a general trick that can be applied in any SE-based algorithm. Specifically, for every SE-based algorithm for delayed feedback (e.g \cite{lancewicki2020learning,Schlisselberg_Cohen_Lancewicki_Mansour_2025}), this will make their additive term be dependent on $\Delta_{avg}$ instead of $\Delta_{max}$.
\end{remark}

\section{Adversarial Analysis}
\begin{theorem} \label{thm:sto regret main}
Let $R_{\texttt{ALG}}$ be the regret of $ALG$ in terms of $T$, $K$, $D$ and possibly $d_{max}$ and $\sigma_{max}$, the regret in the adversarial setting is bounded by:
\begin{align*}
    R_{adv} \le O\roundy{\sqrt{KT\logp{T}} + \logp{K}\sigma_{max} + R_{\texttt{ALG}}}
\end{align*}
\end{theorem}
The $\logp K$ factor can be removed with a slight algorithm modification; see \Cref{appendix:log K}.

\begin{proofsketch}
Fix action $i\in[K]$. Let $\bar{T}$ be the time the algorithm switches to $\texttt{ALG}$. Clearly, the regret after the switch is bounded by $R_{\texttt{ALG}}$, so we focus on the regret up to time $\bar T$.
First, we decompose it to the following three terms:
\begin{align*}
        \underbrace{\E{\sum_{t=1}^{\bar{T}}\squary{l_{a_t}(t) - \sucb{\bar{S}}}}}_{(\ref{eq:first decomposition})} +  \underbrace{\E{\sum_{t=1}^{\tau_i-1}\squary{\sucb{\bar{S}} - l_i(t)}}}_{(\ref{eq:second decomposition})} 
                + \underbrace{\E{\sum_{t=\tau_i}^{\bar{T}}\squary{\sucb{\bar{S}} - l_i(t)}}}_{ (\ref{eq:third decomposition})}
\end{align*}
\refstepcounter{equation}\label{eq:first decomposition}
\refstepcounter{equation}\label{eq:second decomposition}
\refstepcounter{equation}\label{eq:third decomposition}
\!\!\!\!\!\!\!\!\!\!\!\!
where $\bar{S}$ is the value of $S$ when the algorithm switches to \texttt{ALG} and $\tau_i$ is the elimination time of arm $i$. Term~(\ref{eq:first decomposition}) is bounded by $O(\sqrt{KT \logp T} + \logp K \sigma_{max})$ due to the second check of \texttt{BSC}. To bound term~(\ref{eq:second decomposition}), note that $\bar{L}_i$ is an unbiased estimator of $L_i$. By Wald's equation, $\E{L_i(\tilde{S}_i)} = \E{\bar L_i(\tilde{S}_i)} = \E{|\tilde S_i| \omu{i}{\tilde{S}_i}}$, where $\tilde{S}_i$ is the set of rounds in which arm $i$ was observed before elimination. Now, since we have not switched to $\texttt{ALG}$ yet, by the first check of \texttt{BSC} we know that for every realization $S$ of $\tilde{S}_i$,
\begin{align*}
    \omu{i}{S} \geq \olcb{i}{S} - \owidth{S} \geq \oucb{i}{S} - 3\owidth{S} \geq \sucb{S} - 3\owidth{S},
\end{align*}
where in the second inequality we used the fact that $\oucb{i}{S} - \olcb{i}{S} \leq 2\owidth{S}$ for any $S$ and the last inequality is by definition of $\text{ucb}^*$. Multiplying both sides above by $|\tilde S_i|$ gives us
\begin{align*}
    \E{\sum_{t \in \tilde S_i}l_i(t)} \geq \E{|\tilde S_i| \sucb{\tilde{S}_i} - 3\sqrt{2|\tilde S_i|K \logp T}} \geq 
    \E{|\tilde S_i| \sucb{\bar{S}} - 3\sqrt{2 T K \logp T}} 
            \end{align*}
Rearranging the terms above we get that $\E{\sum_{t \in \tilde S_i}\squary{\sucb{\bar{S}} - l_i(t)}} \leq 3\sqrt{2 T K \logp T}$. Hence,
\begin{align*}
    \eqref{eq:second decomposition} = \E{\sum_{t \in \tilde S_i}\squary{\sucb{\bar{S}} - l_i(t)}} + \E{\sum_{t \in [\tau_i\setminus\tilde{S}_i]}\squary{\sucb{\bar{S}} - l_i(t)}} \le 
     3\sqrt{2 T K \logp T} + \sigma_{max}
\end{align*}
where we've used the fact that $\text{ucb}$ decreases over time and $|[\tau_i - 1] \backslash \tilde S_i| \leq \sigma_{max}$.

The core difficulty in the adversarial analysis is bounding (\ref{eq:third decomposition}) — ensuring that an eliminated arm doesn't become 
much better than the active arms, before switching to \texttt{ALG}. Let us further decompose \eqref{eq:third decomposition} to the phases of arm $i$:
\begin{align*}
    \E{\sum_{t=\tau_i}^{\bar{T}}\squary{\sucb{\bar{S}} - l_i(t) }\!} 
        \!\!
        = 
    \E{\sum_{r=1}^{r_i(T)} \sum_{t\in S_i^r} \squary{\sucb{\bar{S}} - l_i(t)}\!}
    \!\!
            = \E{\sum_{r=1}^{r_i(T)} (
            \abs{S_i^r}\sucb{\bar{S}} - L_i(S_i^r))\!}
\end{align*}
The main tool to upper bound the optimality of an eliminated arm is the check in \Cref{alg:third switch main} of \texttt{EAP}. This checks that the estimated loss (using an importance sampling estimator) isn't much higher than the loss observed when the arm was active. Using the condition in \Cref{alg:third switch main} of \texttt{EAP} and by bounding the difference  between the loss estimator of the phase and the actual cost in terms of $N_i^r$ and $\telta_i$ we show the following lemma
 (the proof is deferred to the appendix - see \Cref{lem:single phase error}):

\begin{lemma}
\label{lem:single phase error main}
For every arm $i$ and phase $r$ we have:
\begin{align*}
    \bbE_i^r \squary{\abs{S_i^r}\sucb{\tilde{S}_i} - L_i(S_i^r)} \leq   \frac{3}{8}\telta_i \Nmax{i}{r} - \frac{9}{8}\telta_i \bbE_i^r [\abs{S_i^r}] ,
\end{align*}
where $\bbE_i^r$ is the expectation conditioned on the observed history by the beginning of the $r$th phase of arm $i$.
\end{lemma}

Note that the difference between $\text{ucb}^*$ and the expected loss depends on the relationship between $\abs{S_i^r}$ and $\Nmax{i}{r}$. If the phase finished successfully ($\abs{S_i^r} = \Nmax{i}{r}$), the expected loss exceeds $\text{ucb}^*$. If the phase was erroneous, then we may have $\abs{S_i^r} \ll \Nmax{i}{r}$
, and the expected loss can be better than $\text{ucb}^*$. 
Specifically:
\begin{align*}
    \text{Finished phase:} \quad
    \abs{S_i^r}\sucb{\bar{S}} - \E{L_i(S_i^r)}
        &\leq  - \frac{3}{4}\telta_i \Nmax{i}{r} \\
    \text{Erroneous phase:} \quad
    \abs{S_i^r}\sucb{\bar{S}} - \E{L_i(S_i^r)}
        &\leq  \frac{3}{8}\telta_i \Nmax{i}{r}
\end{align*}

The trick for bounding the sum of these bounds over the phases is that every successfully finished phase compensate for the erroneous phases after it, since the coefficient of $\telta_i \Nmax{i}{r}$ under finished phases is twice as the coefficient for error phases. Since the algorithm halves $\Nmax{i}{r}$ after an error, even $O(\logp T)$ erroneous phases are eventually covered by the last successful phase. The precise argument is by induction and is rather technical. For full details, see the appendix in \Cref{lem:phases losses}.
\end{proofsketch}

To conclude the proof we use the following lemma (proof is deferred to the appendix):
\begin{lemma} \label{lem:sigma bound}
$
    \sigma_{max} \le O\roundy{\min_{S\in [T]} \curly{\abs{S} + \sqrt{D_{\bar{S}}}}}
$.
\end{lemma}
\begin{corollary}
Using \texttt{ALG} as the algorithm from \citet{pmlr-v108-zimmert20a}, we have:
\begin{align*}
    R_{adv} \le \tilde{O}\roundy{\sqrt{KT\logp{T}} + \min_{S\in [T]} \curly{\abs{S} + \sqrt{D_{\bar{S}}}}}
\end{align*}
\end{corollary}

The first term in the regret bound is nearly optimal—up to the $\logp{T}$ factor, it matches the worst-case regret for MAB without delays. The second term, which accounts for the delays, is also tight in general due to the lower bound of \citet{masoudian2022best}. Our bound eliminates altogether the $\Phi^*$ term that appears in the BoBW bound of \citet{masoudian2024best} for the adversarial regime (see \Cref{table: comparison}). This is especially significant whenever $d_{\max}$ is very large (i.e., even if a single delay is large), in which case $\Phi^* = \sqrt{D K^{2/3}}$ while our delay term only scales with $\sqrt{D}$ in the worst case.

\section{Discussion}
We presented a novel algorithm for the delayed-BoBW problem that achieves a near-optimal regret bound simultaneously for both stochastic and adversarial losses. Additionally, our bounds in the stochastic regime improve even compared to algorithms specifically designed for the stochastic case. As mentioned, our algorithm follows the ``adaptive approach" for BoBW—it begins with an algorithm that achieves optimal bounds in the stochastic setting, and upon identifying non-stochastic losses, it switches to an optimal algorithm for the adversarial setting. The alternative perspective \citet{masoudian2022best,masoudian2024best} offers a simpler algorithm but results in weaker bounds. It remains an open question whether algorithms of the latter type can achieve optimality in the delayed scenario. Additionally, while we considered worst-case adversarial delays, future research could explore delays with additional structure (such as i.i.d. or payoff-dependent delays), potentially yielding improved regret bounds.

\section{Acknowledgements}
OS, TL and YM are supported by the European Research Council (ERC) under the European Union’s Horizon 2020 research and innovation program (grant agreement No. 882396), by the Israel Science Foundation and the Yandex Initiative for Machine Learning at Tel Aviv University and by a grant from the Tel Aviv University Center for AI and Data Science (TAD).

\newpage

\bibliographystyle{abbrvnat}
\bibliography{cited}

\newpage
\appendix
\section{Notations}
For a series $S$, we denote $S_{:k}$ to be the first $k$ elements of it. Additionally, $S_{:-1}$ is the series without the last element.

Additionally, for every variable defined inside the algorithm, in the analysis we will add $(t)$ to indicate that we refer to the value of this variable at the end of time $t$. For example, $r_i(t)$ is the value of $r_i$ at the end of time $t$.
\begin{table}[H]
    \centering
    \begin{tabular}{c|l|l}
         $a_t$& chosen arm at step $t$ & \\
         
         $l_i(t)$& loss of arm $i$ at step $t$&\\
         
         $d_t$& delay at step $t$&\\
         
         $\mu_i$& (stochastic) mean loss of arm $i$&\\

         $\bar{T}$& switch to algorithm \texttt{ALG} point&\\
         
         $n_i(S)$& number of pulls of arm $i$ in $S$ &$\abs{s\in S : a_s = i}$\\  

         $m_i(t)$& number of pulls of arm $i$ until time $t$&$\abs{s\in[t]: a_s = i}$ \\

         $\overline{l}_i(t)$ & sample estimator for the loss of arm $i$ in step $t$ &  $l_i(t)\frac{\indicator{a_t=i}}{p_i(t)}$ \\

         $L_i(S)$& total loss of arm $i$ in set $S$& $\sum_{s\in S}l_i(s)$\\

         $\overline{L}_i(S)$& sample estimator for the loss of arm $i$ w.r.t the steps in $S$& $\sum_{s\in S}\overline{l}_i(s)$\\
         
         $\hat{\mu}_i(S)$& average loss of arm $i$ w.r.t the steps in $S$ &$\frac{1}{n_i(S)}\sum_{s\in S : a_s = i}l_i(s)$\\
         
         $\overline{\mu}_i(S)$& sample estimator average loss of arm $i$ w.r.t the steps in $S$ & $\frac{1}{\abs{S}}\overline{L}_i(S)$\\
         
         $\text{width}_i(S)$& average confidence width of arm $i$ w.r.t the steps in $S$ & $\min\curly{1, \sqrt{\frac{2\logp{T}}{n_i(S)}}}$\\

        $\overline{\text{width}}(S)$& estimator confidence width w.r.t the steps in $S$ & $\min\curly{1, \sqrt{\frac{2K\logp{T}}{\abs{S}}}}$\\
         
         $\text{lcb}_i(S)$& average lower confidence bound of arm $i$ w.r.t set $S$ & $\max\curly{\hat{\mu}_i - \width{i}{S},\lcb{i}{S_{:-1}}}$\\

        $\text{ucb}_i(S)$& average upper confidence bound of arm $i$ w.r.t set $S$ & $\min\curly{\hat{\mu}_i + \width{i}{S},\ucb{i}{S_{:-1}}}$\\

         $\overline{\text{lcb}}_i(S)$& estimator lower confidence bound of arm $i$ w.r.t set $S$ & $\max\curly{\omu{i}{S} - \owidth{S},\olcb{i}{S_{:-1}}}$\\

        $\overline{\text{ucb}}_i(S)$& estimator upper confidence bound of arm $i$ w.r.t set $S$ & $\min\curly{\omu{i}{S} + \owidth{S},\oucb{i}{S_{:-1}}}$\\

        $\sucb{S}$ & & $\min_i\curly{\ucb{i}{S},\oucb{i}{S}}$ \\
         
         $p_i^{max}(t_1,t_2)$&maximum pull probability of arm $i$ in the interval& $\max_{t_1\le t \le t_2} p_i(t)$\\
         
         $p_i^{min}(t_1,t_2)$&minimum pull probability of arm $i$ in the interval& $\min_{t_1\le t \le t_2} p_i(t)$\\
         
         $p_i^{max}(t)$& & $p_i^{max}(0, t)$\\
         
         $p_i^{min}(t)$& & $p_i^{min}(0, t)$\\

         $B(t)$ & Set of the steps whose feedback was observed & $\curly{s : s + d_s < t}$  \\

        $B_i^p(t)$ & Set of observed inactive steps that were pulled with prob. $p$& $\curly{s \in B(t) : p_i(s) = p \land s \ge \tau_i}$\\
        $M_i^p(t)$ & Set of inactive steps up to time $t$ that were pulled with prob. $p$& $\curly{\tau_i \le s \le t : p_i(s) = p}$
    \end{tabular}
    \label{tab:notation}
\end{table}

\section{Algorithm}
\begin{algorithm}[H]
\caption{Delayed SAPO Algorithm}
\label{alg:main}
\begin{algorithmic}[1]
\Require Number of arms $K$, number of rounds $T \geq K$, Algorithm \texttt{ALG}.
\State Initialize active arms $\mathcal{A} = \{1, \dots, K\}$, $S = \series{}$
\For{$t = 1, 2, \dots, T$}
    \For{$s\in B \setminus S$}
        \State $S = S + \series{s}$
        \If{not $\texttt{BSC}(S)$ (Procedure~\ref{alg:bsc})}
                \State Switch to \texttt{ALG}.
        \EndIf
        \State $U(t) = \{i \in \mathcal{A}: \hat{\mu}_{i}(S) - 9\width{i}{S} > \sucb{S}\}$ \Comment{Elimination}
        \State $\mathcal{A} = \mathcal{A} \setminus U$
        \For{$i \in U$} \Comment{Initialization for eliminated arms}
            \State Set $\tau_i = t$, $p_i^1 = \frac{1}{2K} + \frac{n_i(S)}{2T}$, $\tilde{S}_i = S$, $\tmu_i = \hat{\mu}_i(S)$, $\telta_i = 8\width{i}{S}$, $\Nmax{i}{1} :=  1280 / (p_i^1 \tilde{\Delta}_i^2)$, $E_i = 0$, $r_i=1$, $S_i^1 = \series{}$, $C_i^{p_i^1\cdot2^{-j}}=\emptyset\;\forall j\in[\logp{T}]$
        \EndFor
    \EndFor
    
    \For{$i \in \bar{\A}$}
        \State $p_i(t), err = \texttt{EAP}(i)$ (Procedure~\ref{alg:eap})
        \If{$err$}
            \State Switch to \texttt{ALG}. 
        \EndIf
    \EndFor
    \State $\forall i\in \A\quad p_i(t) = \left(1 - \sum_{j \in \bar{\A}(t)} p_j(t)\right)/|\mathcal{A}(t)|$
    \State \text{Observe feedback and update variables}
\EndFor

\end{algorithmic}
\end{algorithm}

\begin{algorithm} 
\floatname{algorithm}{Procedure}
\caption{Eliminated Arms Processing (EAP) Subroutine }
\label{alg:eap}
\begin{algorithmic}[1]
\Require Arm $i$

\State $p := p_i^{r_i}$

\While{$B_i^{p} \setminus C_i^{p} \ne \emptyset$}

\For{$s \in B_i^{p} \setminus C_i^{p}$}
    \State $S_i^{r_i} = S_i^{r_i} + \series{s}$, $C_i^{p} = C_i^{p} \cup \curly{s}$
    \If{$\abs{S_i^{r_i}}\tmu_i - \bar{L}\roundy{S_i^{r_i}} \ge \frac{1}{4}\telta_i \Nmax{i}{r_i}$} \Comment{phase error}
        \State $E_i = E_i + 1$, $\Nmax{i}{r_i+1} = \max\curly{\Nmax{i}{1},\frac{1}{2}\Nmax{i}{r_i}}$, $p_i^{r_i+1} = \min\curly{p_i^1,2p_i^{r_i}}$, $S_i^{r_i+1} = \series{}$, $r_i=r_i+1$ \label{alg:error phase}
    \If{$E_i \ge 3\logp{T}$}
        \Return $0$, True \Comment{Switch to adversarial algorithm} \label{alg:third switch}
    \EndIf
                    \State \textbf{break}
    \EndIf
    \If{$\abs{S_i^{r_i}} = \lfloor \Nmax{i}{r_i} \rfloor$} \Comment{phase ended}
        \State $\Nmax{i}{r_i+1} = 2\Nmax{i}{r_i}$, $p_i^{r_i+1} = \frac{1}{2}p_i^{r_i}$ $S_i^{r_i+1} = \series{}$, $r_i=r_i+1$ 
        \State \textbf{break}
        \label{alg:success phase} 
    \EndIf
\EndFor

\State $p := p_i^{r_i}$
\EndWhile

\Return $p_i^{r_i}$, False
\end{algorithmic}
\end{algorithm}

\begin{algorithm} 
\floatname{algorithm}{Procedure}
\caption{Basic Stochastic Checks (BSC) Subroutine}
\label{alg:bsc}
\begin{algorithmic}[1]
\Require Series of processed pulls $S$
\If{\(\exists i \in \mathcal{A} : \omu{i}{S} \not\in [\olcb{i}{S} - \owidth{S}, \oucb{i}{S} + \owidth{S}]\)}\label{alg:first switch}
                \State \Return False 
        \EndIf
        \If{
            $\sum_{s'\in S}\roundy{l_{a_{s'}}(s') - \text{ucb}^*(S)}> 272\sqrt{KT\logp{T}} + 10\sigma_{max}(t)\logp{K}$
        }\label{alg:second switch}
            \State \Return False 
        \EndIf
    \Return True
\end{algorithmic}
\end{algorithm}

\section{General Lemmas}
\begin{lemma}[Freedman's Inequality, Theorem 1.1 in \citet{tropp2011freedmansinequalitymatrixmartingales}] \label{lem:freedman}
Let $\curly{X_k}_{k\geq 1}$ be a real valued martingale difference sequence adapted to a filtration $\curly{F_t}_{t\geq 0}$. If $X_k\leq R$ a.s. Then, for all $t\ge 0$ and $\sigma^2 \ge 0$,
\begin{align*}
    \Pr\squary{\exists k \ge 0: \sum_{i=1}^k X_i \ge t \quad\text{and}\quad \sum_{i=1}^k \E{X_i^2 | F_{i-1}} \le \sigma^2} \le exp\curly{-\frac{t^2}{2\sigma^2 + 2Rt/3}}
\end{align*}
\end{lemma}

\begin{lemma}[Lemma F.4 in \citet{dann2017unifying}]
    \label{lem:dann}
     Let $\{ X_t \}_{t=1}^T$ be a sequence of Bernoulli random and a filtration $\calF_1 \subseteq \calF_2 \subseteq...\calF_T$ with $\bbP(X_t = 1\mid \calF_t) = P_t$, $P_t$ is $\calF_{t}$-measurable and $X_t$ is $\calF_{t+1}$-measurable. Then, for all $t\in [T]$ simultaneously, with probability $1-\delta$,
     \[
        \sum_{k=1}^t X_k \geq \frac{1}{2}\sum_{k=1}^t P_k -\logp{\frac{1}{\delta}}
     \]
\end{lemma}

\begin{lemma}[Consequence of Freedman’s Inequality, e.g., Lemma 27 in \citet{pmlr-v139-efroni21a}]
    \label{lem:cons-freedman}
     Let $\{ X_t \}_{t\geq 1}$ be a sequence of random variables, supported in $[0,R]$, and adapted to a filtration $\calF_1 \subseteq \calF_2 \subseteq...\calF_T$. For any $T$, with probability $1-\delta$,
     \[
        \sum_{t=1}^T X_t \leq 2 \sum_{t=1}^T\bbE[X_t \mid \calF_t] + 4R \logp{\frac{1}{\delta}}.
     \]
\end{lemma}

\begin{lemma} \label{lem:max phases}
For every arm $i$, $r_i(T) \le 7\logp{T}$
\end{lemma}
\begin{proof}
Every phase must finish with \Cref{alg:error phase} or \Cref{alg:success phase}. Let $r_1$ be the number of phases finished with \Cref{alg:error phase}. We have:
\begin{align*}
    \Nmax{i}{\ceil{7\logp{T}}} &\ge \Nmax{i}{1}2^{-r_1}2^{7\logp{T}-r_1} \\
    &= \Nmax{i}{1}2^{7\logp{T}-2r_1}
\end{align*}
From \Cref{alg:third switch}, $r_1 \le 3\logp{T}$. Thus:
\begin{align*}
    \Nmax{i}{\ceil{7\logp{T}}} \ge N_i^1T \ge T
\end{align*}
If the $7\logp{T}$'s round is reached and finished, then the horizon has arrived. Else, the algorithm will switch.
\end{proof}

\begin{lemma} \label{lem:estimator bound}

Let $i$ be some arm, and $S$ be a series of steps. Denote $p_{min} = \min_{s\in S}p_i(s)$. Then, if $\abs{S} \ge \frac{\logp{\frac1\delta}}{p_{min}}$, w.p $1-\delta$: 
\begin{align*}
     \max_k \bar{L}_{S_{:k}} - L_{S_{:k}} &\le 2\sqrt{\frac{\abs{S}\logp{\frac{1}{\delta}}}{p_{min}}} \\
     \max_k L_{S_{:k}} - \bar{L}_{S_{:k}} &\le 2\sqrt{\frac{\abs{S}\logp{\frac{1}{\delta}}}{p_{min}}}
\end{align*}
\end{lemma}
\begin{proof}
\begin{align*}
    \E{\roundy{\overline{l}_i(t)}^2} = l_i(t)^2\frac{p_i(t)}{p_i(t)^2} \le \frac{1}{p_i(t)}
\end{align*}
Thus, for every $k$ the variance is bounded by:
\begin{align*}
    V\roundy{\bar{L}_{S_{:k}} - L_{S_{:k}}} \le \frac{\abs{{S_{:k}}}}{p_{min}} \le \frac{\abs{S}}{p_{min}}
\end{align*}

Using \Cref{lem:freedman} with $\sigma^2 = \frac{\abs{S}}{p_{min}}$ and $R=\frac{1}{p_{min}}$:
\begin{align*}
    \logp{\Pr\squary{\max_k \bar{L}_{S_{:k}} - L_{S_{:k}} \ge 2\sqrt{\frac{\abs{S}\logp{\frac{1}{\delta}}}{p_{min}}}}} \le -\frac{4\frac{\abs{S}\logp{\frac{1}{\delta}}}{p_{min}}}{2\frac{\abs{S}}{p_{min}} + \frac{4}{3}\frac{1}{p_{min}}\sqrt{\frac{\abs{S}\logp{\frac{1}{\delta}}}{p_{min}}}}
\end{align*}

Since $\abs{S} \ge \frac{\logp{\frac1\delta}}{p_{min}}$:
\begin{align*}
    \logp{\Pr\squary{\max_k \bar{L}_{S_{:k}} - L_{S_{:k}} \ge 2\sqrt{\frac{\abs{S}\logp{\frac{1}{\delta}}}{p_{min}}}}} \le -\frac{4\frac{\abs{S}\logp{\frac{1}{\delta}}}{p_{min}}}{2\frac{\abs{S}}{p_{min}} + \frac{4}{3}\frac{\abs{S}}{p_{min}}} \le -\logp{\frac{1}{\delta}}
\end{align*}

Which means that w.p $1-\delta$:
\begin{align*}
    \max_k \bar{L}_{S_{:k}} - L_{S_{:k}} \le 2\sqrt{\frac{\abs{S}\logp{\frac{1}{\delta}}}{p_{min}}}
\end{align*}
Which is exactly the first inequality in the lemma. The second has the same proof.
\end{proof}

\section{Stochastic}
\subsection{Good Event}
\begin{definition}
Let $G_{sto}$ be the event that:
\begin{align*}
    \forall i\in [K],n\le \abs{\tilde{S}_i} &\quad|\mu_i - \hat{\mu}_i(S_{:n}) | \le \width{i}{S_{:n}} \\
    \forall i\in [K],n\le \abs{\tilde{S}_i}&\quad \abs{\mu_i - \omu{i}{S_{:n}}} \le \owidth{S_{:n}} \\
    \forall i\in[K],n\le\abs{\tilde{S_i}},S\in\roundy{\curly{\roundy{\tilde{S}_i}_{:n}} \cup \curly{S_i^r|r\in [r_i(\bar{T})]}}&\quad \frac{1}{2}\sum_{s\in S}p_i(s) - 3\logp{T} \le n_i(S) \le 2\sum_{s\in S}p_i(s) + 12\logp{T} \\
    \forall i\in[K]&\quad E_i(\bar{T}) \le \E{E_i(\bar{T})} + \sqrt{7}\logp{T} \\
    \forall S \quad& \sum_{s\in S}[l_{a_s}(s) - l_{a^*}(s)] - \E{\sum_{s\in S}[l_{a_s}(s) - l_{a^*}(s)]} \le 3\sqrt{T\logp{T}}
\end{align*}
If the delays are stochastic, also:
\begin{align*}
    \sigma_{max} \le 2\E{d} + 8\logp{T}
\end{align*}
\end{definition}

\begin{lemma} \label{lem:avg bound}
For every state of $S$ and arm $i$, w.p $1-\frac{2}{T}$:
\begin{align*}
    |\mu_i - \hat{\mu}_i(S) | &\le \width{i}{S}
\end{align*}
\end{lemma}
\begin{proof}
Directly from Equation 1.6 of \citet{slivkins2024introductionmultiarmedbandits}.
\end{proof}

\begin{lemma} \label{lem:active estimator}
With probability $1-\frac{2}{T}$, For every arm $i$ and $n \le \abs{\tilde{S}_i}$:
\begin{align*}
    \abs{\omu{i}{S_{:n}} - \mu_i} \le \owidth{S_{:n}}
\end{align*}
\end{lemma}
\begin{proof}
Denote $S = S_{:n}$ for brevity.

If $\abs{S} \le 3K\logp{T}$ then $\owidth{S} \ge 1$ and the bound is trivial.

Notice that since $n\le \abs{\tilde{S}_i}$, for $s\in S$ we have $p_i(s) \ge \frac{1}{K}$. If $\abs{S} \ge 3K\logp{T}$, w.p $1 - \frac{2}{T^3}$ from \Cref{lem:estimator bound}:
\begin{align*}
    \abs{S}\abs{\omu{i}{S} - \mu_i} &\le 2\sqrt{3K\abs{S}\logp{T}} \\
    \abs{\omu{i}{t} - \mu_i} &\le 2\sqrt{\frac{3K\logp{T}}{\abs{S}}}
\end{align*}
Union bound on all arms and $n\le \abs{\tilde{S}_i}$ we get the desired results.
\end{proof}

\begin{lemma} \label{lem:pulls bound}
With probability $1-\frac{2}{T}$
\begin{align*}
    \forall i\in[K],S\in\roundy{\curly{\roundy{\tilde{S}_i}_{:k}}}\lor S\in \curly{S_i^r|r\in r_i(\bar{T})}&\quad \frac{1}{2}\sum_{s\in S}p_i(s) - 3\logp{T} \le n_i(S) \le 2\sum_{s\in S}p_i(s) + 12\logp{T}
\end{align*}
\end{lemma}
\begin{proof}
Notice that $n_i(S)$ is a sum of bernoulli variables. Fix $S$ and $i$, from \Cref{lem:dann,lem:cons-freedman}, w.p $1-\frac{2}{T^3}$:
\begin{align*}
    \frac{1}{2}\sum_{s\in S}p_i(s) - 3\logp{T} \le n_i(t) \le 2\sum_{s\in S}p_i(s) + 12\logp{T}
\end{align*}
Union bound for all the options for $S$ and $i$ gives us the desired results. 
\end{proof}

\begin{lemma} \label{lem:error prob}
For every arm, w.p $1-\frac{1}{T}$:
\begin{align*}
    E_i(\bar{T}) \le \E{E_i(\bar{T})} + \sqrt{7}\logp{T}
\end{align*}
\end{lemma}
\begin{proof}
    Using Hoeffding, with probability $1 - \frac{1}{T^2}$:
\begin{align*}
    E_i(\bar{T}) \le \E{E_i(\bar{T})} + \sqrt{\frac{1}{2}r_i(\bar{T})\logp{T^2}}
\end{align*}
From \Cref{lem:max phases}:
\begin{align*}
    E_i(\bar{T}) \le \E{E_i(\bar{T})} + \sqrt{7}\logp{T}
\end{align*}
Union bound for all arms concludes the proof.
\end{proof}

\begin{lemma} \label{lem:pseudo to regret}
For every state of $S$, w.p $1-\frac{1}{T}$:
\begin{align*}
    \sum_{s\in S}[l_{a_s}(s) - l_{a^*}(s)] - \E{\sum_{s\in S}[l_{a_s}(s) - l_{a^*}(s)]} \le 3\sqrt{T\logp{T}}
\end{align*}
\end{lemma}
\begin{proof}
From \Cref{lem:freedman}, with $\sigma^2 = T$ and $R = 1$:
\begin{align*}
    \Pr\squary{\exists S: \sum_{s\in S}[l_{a_s}(s) - l_{a^*}(s)] - \E{\sum_{s\in S}[l_{a_s}(s) - l_{a^*}(s)]} \ge 3\sqrt{T\logp{T}}} &\le exp\roundy{-\frac{9T\logp{T}}{2T + 2\sqrt{T\logp{T}}}} \\
    &\le exp\roundy{-\logp{T}} = \frac{1}{T}
\end{align*}
\end{proof}

\begin{lemma} \label{lem:missing to delay mean}
If the delays are stochastic, with probability $1 - \frac{1}{T}$:
\begin{align*}
    \sigma_{max} \le 2\E{d} + 8\logp{T}
\end{align*}
\end{lemma}
\begin{proof}
For every $t\le T$:
\begin{align*}
    \E{\sigma(t)} &= \sum_{s=1}^t \Pr[d > t-s]\\
    &= \sum_{s=1}^t \sum_{i=t-s+1}^{\infty} \Pr[d=i] 
    \\
    &\le \sum_{i=1}^\infty i\Pr[d=i]\\
    &= \E{d}
\end{align*}

Fix some $t \le T$. From \Cref{lem:cons-freedman}, w.p $1 - \frac{1}{T^2}$:
\begin{align*}
    \sigma(t) \le 2\E{d} + 8\logp{T}
\end{align*}

Union bound for all $t$ concludes the proof.
\end{proof}

\begin{corollary} \label{cor:Gsto prob}
$G_{sto}$ happens w.p $1-\frac{9}{T}$
\end{corollary}
\begin{proof}
Union bound of \Cref{lem:active estimator,lem:avg bound,lem:pulls bound,lem:error prob,lem:pseudo to regret,lem:missing to delay mean}
\end{proof}

\subsection{Regret Analysis}

\begin{lemma} \label{lem:optimal not eliminated}
Assume $G_{sto}$, the optimal arm $i^*$ will not be evicted.
\end{lemma}
\begin{proof}
Assume by contradiction that arm $i^*$ was evicted. Namely, there is $S$ such that:
\begin{align*}
    \text{ucb}^*(S) < \hat{\mu}_{i^*}(S) - 9\width{i^*}{S}
\end{align*}
From $G_{sto}$:
\begin{align*}
    \sucb{S} < \mu_{i^*}(S) - 9\width{i^*}{S} < \mu^*
\end{align*}
From the definition of $\text{ucb}^*$ and $G_{sto}$, there is an arm $i$ such that:
\begin{align*}
    \mu_i \le \sucb{S} < \mu^*
\end{align*}
Which contradicts the fact that $i^*$ is optimal.
\end{proof}

\begin{lemma} \label{lem:delta bound}
Assume $G_{sto}$, For every two arms $i_1,i_2$ and series $n \le \min\curly{\abs{\tilde{S}_{i_1}}, \abs{\tilde{S}_{i_2}}}$:
\begin{align*}
    \width{i_1}{S_{:n}} \le 10\width{i_2}{S_{:n}}
\end{align*}

Additionally, for every arm $i$:
\begin{align*}
    \telta_i \le \Delta_i \le 2\telta_i
\end{align*}
\end{lemma}
\begin{proof}
Denote $S = S_{:n}$ for brevity. 

Since $i_1$ was not eliminated:
\begin{align*}
    \hat\mu_{i_1}({S}) - 9\width{i_1}{S} \le \sucb{S} \le \sucb{S} \le \ucb{i_2}{S} \le  \hat{\mu}_{i_2}(S) + \width{i_2}{S}
\end{align*}

From $G_{sto}$:
\begin{align*}
    \mu_i - 10\width{i_1}{S} &\le \mu^* + 2\width{i_2}{S} \\
    \Delta_{i_1} &\le 10\width{i_1}{S} + 2\width{i_2}{S}
\end{align*}

Notice that since $i_1$ and $i_2$ were not evicted in the steps of $S$, we have $\E{n_{i_1}(S)} = \E{n_{i_2}(S)}$. Thus:
\begin{align*}
    n_{i_1}(S) &\le 2\E{n_{i_1}(S)} + 12\logp{T} \\
    \E{n_{i_1}(S)} &\ge \frac{1}{2}n_{i_1}(S) - 6\logp{T} \\
    n_{i_2}(S) &\ge \frac{1}{2}\E{n_{i_2}(S)} - 3\logp{T} \\
    \E{n_{i_2}(S)} &\le 2n_{i_2}(S) + 6\logp{T}  \\
    \frac{1}{2}n_{i_1}(S) - 6\logp{T} &\le 2n_{i_2}(S) + 6\logp{T} \\
    \frac{1}{4}n_{i_1}(S) - 24\logp{T} &\le n_{i_2}(S) \\
\end{align*}

If $n_{i}(S) \ge 192\logp{T}$:
\begin{align*}
    \frac{1}{4}n_{i_1}(S) - \frac{1}{8}n_{i_1}(S) &\le n_{i_2}(S) \\
    \frac{1}{8}n_{i_1}(S) &\le n_{i_2}(S) \\
    \sqrt{\frac{2\logp{T}}{n_{i_2}(S)}} &\le 3\sqrt{\frac{2\logp{T}(S)}{n_{i_1}(S)}} \\
    \width{i_2}{S} &\le 3\width{i_1}{S} \\
    \Delta_{i_1} &\le 16\width{i_1}{S}
\end{align*}

If $n_{i}(S) \le 192\logp{T}$:
\begin{align*}
    \width{i_1}{S} &\ge \sqrt{\frac{2\logp{T}}{192\logp{T}}} \ge \frac{1}{10} \\
    \Delta_{i_1} &\le 1 \le  10\width{i_1}{S} \\
    \width{i_2}{S} &\le 1 \le  10\width{i_1}{S}
\end{align*}

In both cases:
\begin{align*}
    \Delta_{i_1} \le 16\width{i_1}{S} \le 16\width{i_1}{\tilde{S}_{i_1}} = 2\telta_{i_1}
\end{align*}

Additionally, from $G_{sto}$:
\begin{align*}
    \mu_{i_1} - 8\width{i_1}{\tilde{S}_{i_1}} &\ge \hat\mu_{i_1}({\tilde{S}_{i_1}}) - 9\width{i_{i_1}}{\tilde{S}_{i_1}} > \sucb{\tilde{S}_{i_1}} \ge \mu^* \\
    \Delta_{i_1} &\ge 8\width{i_1}{\tilde{S}_{i_1}} = \telta_{i_1}
\end{align*}

\end{proof}

\begin{lemma} \label{lem:half of set}
Let $N$ be a set over $\rr$ with average $\mu$. Then, at most half of $N$ are greater than $2\mu$.
\end{lemma}
\begin{proof}
Let $X$ be a r.v sampled from $N$ with uniform distribution. Easy to see that $\E{X} = \mu$. From Markov inequality:
\begin{align*}
 Pr[X \ge 2\mu] \le \frac{\mu}{2\mu} = \frac{1}{2}
\end{align*}
\end{proof}

\begin{lemma}[restatement of \Cref{lem:deltas are monotone main}]
    \label{lem:deltas are monotone}
    If arm $i_1$ was eliminated before $i_2$ then,
    \begin{align*}
        \Delta_{i_2} \leq 20\Delta_{i_1}
    \end{align*}
\end{lemma}

\begin{proof}
Let $i_1$ and $i_2$ be two arms such that $i_1$ was evicted before $i_2$. From \Cref{lem:delta bound}:
\begin{align*} 
    \Delta_{i_2} \le 16\width{i_2}{\tilde{S}_{i_2}} \le 16\width{i_2}{\tilde{S}_{i_1}} \le 160\width{i_1}{\tilde{S}_{i_1}} \le 20\Delta_{i_1}
\end{align*}
\end{proof}

\begin{lemma} \label{lem:delta avg}
Assume $G_{sto}$, we have:
\begin{align*}
    \sum_i p_i^{max}(\tau_i)\Delta_i &\le \Delta_{max}\logp{K}\\
    \sum_i p_i^{max}(\tau_i)\Delta_i &\le 42\Delta_{avg}\logp{K}
\end{align*}
\end{lemma}
\begin{proof}

Let $\A'$ be the state of $\A$ such that $\abs{\A} = \frac{K}{2}$ (namely, the last half active arms). We show that:
\begin{align*}
    \forall i\in \A'\quad \Delta_i \le 40\Delta_{avg}
\end{align*}

If $\Delta_i \le 2\Delta_{avg}$ it is trivial. Otherwise, from \Cref{lem:half of set} there is an arm $j \notin \A'$ such that $\Delta_i \le 2\Delta_{avg}$. From \Cref{lem:deltas are monotone}, $\Delta_i \le 20\Delta_j \le 40\Delta_{avg}$.

Notice that if arm $i$ was the $j$th evicted arm, we have:
\begin{align*}
    p_i^{max}(\tau_i) \le \frac{1}{j}
\end{align*}

Which means:
\begin{align*}
    \sum_i p_i^{max}(\tau_i)\Delta_i &= \sum_{i\notin \A'} p_i^{max}(\tau_i)\Delta_i + \sum_{i\in \A'} p_i^{max}(\tau_i)\Delta_i \\
    &\le \sum_{i\notin \A'} \frac{\Delta_i}{\frac{K}{2}} + \sum_{j=1}^{\frac{K}{2}}\frac{40\Delta_{avg}}{j}\\
    &\le 2\Delta_{avg} + 40\Delta_{avg}\logp{K}
\end{align*}

\end{proof}

\begin{lemma} \label{lem:regret active}
Assume $G_{sto}$, then for any arm $i$ and step $t < \tau_i$:
\begin{align*}
   m_i(t)\Delta_i \le \frac{515\logp{T}}{\Delta_i} + 2\sigma(t) \Delta_i p_i^{max}(\tau_i)
\end{align*}
\end{lemma}
\begin{proof}
Let $S$ be the set of observed pulls at time $t$. From \Cref{lem:delta bound}:
\begin{align*}
    \Delta_i &\le 16\sqrt{\frac{2\logp{T}}{n_i(S)}}\\
    n_i(S) &\le \frac{512\logp{T}}{\telta^2}
\end{align*}

At step $t$, there are $\sigma(t)$ missing pulls. When those missing pulls were pull, the probability of arm $i$ was bounded by $p_i^{max}(\tau_i-1)$ (as it is its general maximum probability). Thus, from $G_{sto}$, we have:
\begin{align*}
    m_i(t) \le n_i(S) + 2\sigma(t) p_i^{max}(\tau_i-1) + 3\logp{T}
\end{align*}
Thus:
\begin{align*}
    m_i(t)\Delta_i \le \frac{515\logp{T}}{\Delta_i} + 2\sigma(t) p_i^{max}(\tau_i)\Delta_i
\end{align*}
\end{proof}

\begin{lemma} \label{lem:regret nonactive}
Assume $G_{sto}$ and that $K \le \frac{T}{12\logp{T}}$, then for any arm $i$ and step $\tau_i < t \le \bar{T}$:
\begin{align*}
    \roundy{m_i(t) - m_i(\tau_i)}\Delta_i &\le \frac{71728\logp{T}}{\Delta_i} + 8p_i^{max}(\tau_i)\sigma_{max}(t)\Delta_i
\end{align*}
\end{lemma}
\begin{proof}
From the algorithm defintion we have:
\begin{align}\label{eq:inactive pulls}
    m_i(t) - m_i(\tau_i) = \sum_{r=1}^{r_i(t)}n_i(S_i^r)  + \sum_{j=0}^{\logp{T}}n_i(M_i^{p_i^12^{-j}}(t) \setminus C_i^{p_i^12^{-j}}(t))
\end{align}

From $G_{sto}$, for every $i$ and every phase $r_i$:
\begin{align*}
    \E{n_i(S_i^r)} = p_i^r\Nmax{i}{r} = p_i^1\Nmax{i}{1} = \frac{1280}{\telta_i^2}
\end{align*}

From \Cref{lem:max phases} and $G_{sto}$:
\begin{align*}
\E{\sum_{r=1}^{r_i(t)}n_i(S_i^r)} &\le 7\logp{T}\frac{1280}{\telta_i^2} \le \frac{8960\logp{T}}{\telta_i^2} \\
    \sum_{r=1}^{r_i(t)}n_i(S_i^r) &\le \frac{17920\logp{T}}{\telta_i^2} + 3\logp{T} \le \frac{17923\logp{T}}{\telta_i^2}
\end{align*}

Let $s\le t$ be the last pull of a phase w.p $p$, which means that $B_i^p(s) \setminus C_i^p(s) = \emptyset$, so $\abs{M_i^p(t) \setminus C_i^p(t)} = \abs{M_i^p(s) \setminus C_i^p(s)} \le \sigma(s) \le \sigma_{max}(t)$. Then: 
\begin{align*}
    \E{n_i(M_i^p \setminus C_i^p)} &\le p\sigma_{max}(t) \\
    \sum_{j=0}^{\logp{T}}\E{n_i(M_i^{p_i^1\cdot2^{-j}} \setminus C_i^{p_i^12^{-j}})} &\le \sum_{j=0}^{\logp{T}}p_i^12^{-j}\sigma_{max}(t) \le 2p_i^1\sigma_{max}(t) \\
    \sum_{j=0}^{\logp{T}}n_i(M_i^{p_i^12^{-j}} \setminus C_i^{p_i^12^{-j}}) &\le 4p_i^1\sigma_{max}(t) + 3\logp{T} \tag{$G_{sto}$}
\end{align*}
Thus, from \Cref{eq:inactive pulls}:
\begin{align*}
    m_i(t) - m_i(\tau_i) \le \frac{17926\logp{T}}{\telta_i^2} + 4p_i^1\sigma_{max}(t)
\end{align*}

From $G_{sto}$ and the assumption that $K \le \frac{T}{12\logp{T}}$:
\begin{align*}
    \frac{1}{K} &\le p_i^{max}(\tau_i)\\
    n_i(\tilde{S}_i) &\le 2\abs{\tilde{S}_i}p_i^{max}(\tau_i) + 12\logp{T} \le 2Tp_i^{max}(\tau_i)+\frac{T}{K} \le 3Tp_i^{max}(\tau_i) \\
    p_i^1 &= \frac{1}{2K} + \frac{n_i(\tilde{S}_i)}{2T} \le \frac{p_i^{max}(\tau_i)}{2} + \frac{3p_i^{max}(\tau_i)}{2} = 2p_i^{max}(\tau_i) \\
    m_i(t) - m_i(\tau_i) &\le \frac{17932\logp{T}}{\telta_i^2} + 8p_i^{max}(\tau_i)\sigma_{max}(t)
\end{align*}

From \Cref{lem:delta bound}:
\begin{align*}
    m_i(t) - m_i(\tau_i) &\le \frac{71728\logp{T}}{\Delta_i^2} + 8p_i^{max}(\tau_i)\sigma_{max}(t) \\
    \roundy{m_i(t) - m_i(\tau_i)}\Delta_i &\le \frac{71728\logp{T}}{\Delta_i} + 8p_i^{max}(\tau_i)\sigma_{max}(t)\Delta_i
\end{align*}

\end{proof}

\begin{corollary} \label{cor:sto regret}
Assume $G_{sto}$, for every $t\le \bar{T}$:
\begin{align*}
    \sum_{i=1}^K \Delta_im_i(t) &\le \sum_{i=1}^K\frac{72243\logp{T}}{\Delta_i} + 10\sigma_{max}(t) \Delta_{max}\logp{K} \\
    \sum_{i=1}^K \Delta_im_i(t) &\le \sum_{i=1}^K\frac{72243\logp{T}}{\Delta_i} + 420\sigma_{max}(t) \Delta_{avg}\logp{K}
\end{align*}
\end{corollary}
\begin{proof}
If $K \le \frac{T}{12\logp{T}}$, it follows directly from \Cref{lem:regret nonactive,lem:regret active,lem:delta avg}. Else, we have:
\begin{align*}
    \sum_{i=1}^K \Delta_im_i(t) \le T \le 12K\logp{T} \le \sum_{i=1}^K\frac{12\logp{T}}{\Delta_i}
\end{align*}
\end{proof}

\subsection{With high probability the algorithm doesn't switch}
\begin{lemma} \label{lem:switch1}
Assume $G_{sto}$, for every arm $i$ and $S \subseteq \tilde{S}_i$: 
\begin{align*}
    \omu{i}{S} \in [\olcb{i}{S} - \owidth{S}, \oucb{i}{S} + \owidth{S}]
\end{align*}
\end{lemma}
\begin{proof}
From $G_{sto}$:
\begin{align*}
    \olcb{i}{S} &\le \mu_i \le \omu{i}{S} + \owidth{S} \\
    \oucb{i}{S} &\ge \mu_i \ge \omu{i}{S} - \owidth{S}
\end{align*}
\end{proof}

\begin{lemma} \label{lem:switch2}
Assume $G_{sto}$, for every state of $S$ at time $t$:
\begin{align*}
    \sum_{s\in S}[l_{a_s}(s) - \sucb{S}] \le 272\sqrt{KT\logp{T}} + 30\sigma_{max}(t)\logp{K}
\end{align*}
\end{lemma}
\begin{proof}
From \Cref{cor:sto regret}:
\begin{align*}
    \E{\sum_{s\in S}[l_{a_s}(s) - l_{a^*}(s)]} &\le \sum_{i=1}^K\frac{72243\logp{T}}{\Delta_i} + 30\sigma_{max}(t) \Delta_{i}\\
    &\le \sum_{i\;\text{s.t}\; \Delta_i \le 269\sqrt{\frac{\logp{T}}{KT}}}269\sqrt{\frac{T\logp{T}}{K}} + \sum_{i\;\text{s.t}\; \Delta_i \ge 269\sqrt{\frac{\logp{T}}{KT}}}\frac{72243\logp{T}}{\Delta_i} + 30\sigma_{max}(t) \\
    &\le 269\sqrt{KT\logp{T}} + 30\sigma_{max}(t)\logp{K}
\end{align*}

From $G_{sto}$, for every state of $S$:
\begin{align*}
     \sum_{s\in S}[l_{a_s}(s) - \sucb{S}] \le \sum_{s\in S}[l_{a_s}(s) - l_{a^*}(s)] \le 272\sqrt{KT\logp{T}} + 30\sigma_{max}(t)\logp{K}
\end{align*}
\end{proof}

\begin{lemma} \label{lem:switch3}
Assume $G_{sto}$, for every arm $i$, $E_i(\bar{T}) \le 3\logp{T}$.
\end{lemma}
\begin{proof}
Fix phase $r$. We will show that w.p $\frac{31}{32}$, 
\begin{align*}
    \abs{S_i^r}\tmu_i - \bar{L}_i^r &\le \frac{1}{4}\telta \Nmax{i}{r}
\end{align*}
If $\abs{S_i^r} \le \frac{1}{4}\telta_i \Nmax{i}{r}$, this is trivial. Otherwise, it means that:
\begin{align*}
    \abs{S_i^r} &\ge \frac{1}{4}\telta_i \Nmax{i}{r} \\
    &= \frac{1}{4}\telta_i\frac{p_i^1N_i^1}{p_i^r}\\
    &= \frac{1}{4}\telta_i\frac{1280}{p_i^r\telta_i^2}\\
    &= \frac{320}{p_i^r\telta_i} \\
    &\ge \frac{40}{p_i^r} \tag{$\telta_i \le 8$}
\end{align*}

Now we can use \Cref{lem:estimator bound}, w.p $\frac{31}{32}$, as the inequality $\abs{S_i^r} \ge \frac{\logp{32}}{p_i^r}$ is satisfied. We have: 
\begin{align}
    \E{L_i(S^r_i)} - \bar{L}_i(S_i^r) &\le 2\sqrt{\frac{5\Nmax{i}{r}}{p_i^r}} \notag \\
    &= 2\sqrt{\frac{5\Nmax{i}{r}^2\telta_i^2}{1280}} \notag\\
    &= \frac{1}{8}\Nmax{i}{r}\telta_i \label{eq:phase loss error}
\end{align}

Additionally, from $G_{sto}$:
\begin{align*}
    \tmu_i - \mu_i  &\le \width{i}{\tilde{S}_i} = \frac{1}{8}\telta \tag{$G_{sto}$} \\
    \abs{S_i^r}\tmu_i - \E{L_i^r} &\le \frac{1}{8}\telta \abs{S_i^r} \le \frac{1}{8}\telta \Nmax{i}{r} \\
    \abs{S_i^r}\tmu_i - \bar{L}_i^r &\le \frac{1}{4}\telta \Nmax{i}{r} \tag{\Cref{eq:phase loss error}}
\end{align*}
All of the above is true to all states throughout the phase (since \Cref{lem:estimator bound} is true for $\max_k$).

This means that in every phase \Cref{alg:error phase} happens w.p $\frac{1}{32}$. Thus:
\begin{align*}
    \E{E_i(\bar{T})} &= \frac{r_i(\bar{T})}{32} \\
    &\le \frac{7\logp{T}}{32} \tag{\Cref{lem:max phases}}
\end{align*}
From $G_{sto}$:
\begin{align*}
    E_i(\bar{T}) \le \E{E_i(\bar{T})} + \sqrt{7}\logp{T} \le 3\logp{T}
\end{align*}
\end{proof}

\begin{corollary} \label{cor:no switch}
Assume $G_{sto}$,  $T = \bar{T}$.
\end{corollary}
\begin{proof}
Directly from \Cref{lem:switch1,lem:switch2,lem:switch3}
\end{proof}

\subsection{Conclusion}
\begin{theorem} \label{thm: sto regret}
For adversarial delays:
\begin{align*}
    \mathcal{R}_{sto} \le O\roundy{\sum_{i=1}^K \frac{\logp{T}}{\Delta_i} + \sigma_{max}\Delta_{avg}\logp{K}}
\end{align*}
For stochastic delays we can also say:
\begin{align*}
    \mathcal{R}_{sto} \le O\roundy{\sum_{i=1}^K \frac{\logp{T}}{\Delta_i} + \E{d}\Delta_{avg}\logp{K}}
\end{align*}
\end{theorem}
\begin{proof}
Assume $G_{sto}$, from \Cref{cor:no switch,cor:sto regret}, The above is true with probability $1-\frac{2}{T}$.

From \Cref{cor:Gsto prob}, this is asymptotically true even without the assumption of $G_{sto}$.
\end{proof}

\section{Adversarial}

\begin{lemma} \label{lem:light tail mean}
Let $X$ be a random variable such that for every $x\ge 0$ there is some $a > 0$ such that:
\begin{align*}
    F_X(x) \ge 1-e^{-x/a}
\end{align*}
Then:
\begin{align*}
    \E{X} \le a
\end{align*}
\end{lemma}
\begin{proof}
We use the CDF representation of the expectation:
\begin{align*}
    \E{X} &=  \int_0^\infty(1-F(x))dx + \int_{-\infty}^0 F(x)dx \\
    &\le \int_0^\infty (1-F(x))dx\\
    &\le \int_0^\infty e^{-x/a}dx \\
    &= \squary{-ae^{-x/a}}^{\infty}_0\\
    &= a
\end{align*}
\end{proof}

\begin{lemma}[restatement of \Cref{lem:single phase error main}]
\label{lem:single phase error}
Let $\mathcal{H}_{i}^r$ be the history (i.e., chosen actions) of rounds that are observed by the begining of the $r$ phase of arm $i$. Denote $\bbE_i^r[\cdot] = \bbE[\cdot \mid \mathcal{H}_{i}^r]$ and $\Pr_i^r[\cdot] = \Pr[\cdot \mid \mathcal{H}_{i}^r]$.
For every arm $i$ and phase $r$ we have: 
\begin{align*}
    \bbE_i^r \squary{\abs{S_i^r}\sucb{\tilde{S}_i} - L_i(S_i^r)} \leq   \frac{3}{8}\telta_i \Nmax{i}{r} - \frac{9}{8}\telta_i \bbE_i^r [\abs{S_i^r}] 
\end{align*}
\end{lemma}
\begin{proof}
Let $S$ be the sequence of $N_i^r$ observed rounds starting from the beginning of $r$ phase of arm $i$, so that the first $|S_i^r|$ rounds in $S$ are exactly $S_i^r$. Let also $X_1,...,X_{N_i^r} \overset{i.i.d}{\sim} \text{Bernoulli}(p_i^r)$.
From \Cref{lem:estimator bound}, if $N_i^r \ge \frac{m}{p_i^r}$, we have for every $m > 0$, conditioned on the history $\mathcal{H}_{i}^r$, with probability of at least $1 - e^{-m}$:
\begin{align*}
    \max_{k}\sum_{t\in S_{:k}} 
        \left[ \frac{X_{t}l_t}{p_{i}^r} -  l_{i}(t)\right ]
        &\le 2\sqrt{\frac{m N_i^r}{p_i^r}} \\
    &\le 2\sqrt{\frac{m\Nmax{i}{r}}{\frac{p_i^1\Nmax{i}{1}}{\Nmax{i}{r}}}} \\
    &= 2\sqrt{\frac{m\Nmax{i}{r}^2\telta_i^2}{1280}}\\
    &\le \frac{m\Nmax{i}{r}\telta_i}{\sqrt{320}}
\end{align*}
Now, note that for $k = |S_i^r|$, $\sum_{t\in S_{:k}} \big[\frac{X_{t}l_t}{p_i^{r}} - l_{i}(t)  \big ]$
distributes exactly like $  L_i(S_i^r) - \bar{L}_i(S_i^r)$. Thus, conditioned on the history $\mathcal{H}_{i}^r$, with probability of at least $1 - e^{-m}$:

\begin{align*}
    \bar{L}_i(S_i^r) - L_i(S_i^r)  &\le \abs{S_i^r}
        \leq 
            \frac{m\Nmax{i}{r}\telta_i}{\sqrt{320}}
\end{align*}

If $N_i^r \leq \frac{m}{p_i^r}$:
\begin{align*}
    \bar{L}_i(S_i^r) - L_i(S_i^r)  &\le N_i^r \\
    &\le \frac{m}{p_i^r}\\
    &= \frac{mN_i^r}{p_i^1N_i^1}\\
    &= \frac{m\Nmax{i}{r}\telta_i^2}{1280}\\
    &\le \frac{m\Nmax{i}{r}\telta_i}{160} \tag{$\telta_i \le 8$}\\
    &\le \frac{m\Nmax{i}{r}\telta_i}{\sqrt{320}}
\end{align*}

From \Cref{lem:light tail mean}:
\begin{align}
     \bbE_i^r\squary{\bar{L}_i(S_i^r) - L_i(S_i^r) } \le \frac{\Nmax{i}{r}\telta_i}{\sqrt{320}} \le \frac{1}{8}\Nmax{i}{r}\telta_i
    \label{eq: phase loss concentration}
\end{align}

Notice that if the phase was not finished it means that $\abs{S_i^r}\tmu_i - \bar{L}_i^{r} \le \frac{1}{4}\telta \Nmax{i}{r}$. 
Given that we have:

\begin{align*}
    \bbE_{i}^{r}\left[\abs{S_{i}^{r}}\sucb{\tilde{S}_{i}} - \bar L_{i}(S_{i}^{r})
        +\frac{9}{8}\telta_{i}|S_{i}^{r}|\right] 
            & =\bbE_{i}^{r} \left[\abs{S_{i}^{r}}\roundy{\sucb{\tilde{S}_{i}} + 9\width i{\tilde{S}_{i}}} - \bar L_{i}(S_{i}^{r})\right]
            \\
            & < \bbE_{i}^{r} \left[\abs{S_{i}^{r}}\tilde{\mu}_{i} - \bar L_{i}(S_{i}^{r}) \right]
                \tag{Elimination inequality}\\
            & \leq \frac{1}{4}\telta_{i}\Nmax ir
\end{align*}

Combining the above with \Cref{eq: phase loss concentration} completes the proof. 

\end{proof}

\begin{lemma} \label{lem:phases losses}
For every arm $i$ and phase $r$,
\begin{align*}
    \bbE_i^r\squary{\sum_{r'=r}^{r_i(T)}\squary{\abs{S_i^{r'}}\sucb{\tilde{S}_i} - L_i(S_i^{r'})}} \le \frac{3(r_i(T)-r-2)}{8}\telta_i\Nmax{i}{1} + \frac{3}{4}\telta_i \Nmax{i}{r}
\end{align*}
\end{lemma}
\begin{proof}
We will prove using reverse induction on $r$.

For $r=r_i(T)$ (namely, after all phases) we have:
\begin{align*}
\bbE_i^r\squary{\sum_{r'=r_i(T)}^{r_i(T)}\squary{\abs{S_i^{r'}}\sucb{\tilde{S}_i} - L_i^{r'}}} &= 0\\
\frac{3(0-2)}{8}\telta_i \Nmax{i}{1} + \frac{3}{4}\telta_i \Nmax{i}{r} &\ge \frac{3(0-2)}{8}\telta_i \Nmax{i}{1} + \frac{3}{4}\telta_i \Nmax{i}{1} = 0
\end{align*}

Assume true for $r+1$, we prove for $r$. 

If \Cref{alg:success phase} was triggered it means that $\abs{S_i^r} = \Nmax{i}{r}$ and $\Nmax{i}{r+1} = 2\Nmax{i}{r}$. We have:
\begin{align*}
    \bbE_i^r\squary{\abs{S_i^{r}}\sucb{\tilde{S}_i} - L_i^r} &\le   \frac{3}{8}\telta_i \Nmax{i}{r} - \frac{9}{8}\telta_i \Nmax{i}{r}  
    \tag{\Cref{lem:single phase error}}
    \\
    \bbE_i^r\squary{\sum_{r'=r}^{r_i(T)}\squary{\abs{S_i^{r'}}\sucb{\tilde{S}_i} - L_i(S_i^{r'})}} &= \bbE_i^r\squary{\sum_{r'=r+1}^{r_i(T)}\squary{\abs{S_i^{r'}}\sucb{\tilde{S}_i} - L_i(S_i^{r'})}} + \bbE_i^r\squary{\abs{S_i^{r}}\sucb{\tilde{S}_i} - L_i^r} \\
    &\le \frac{3(r_i(T)-r-3)}{8}\telta_i\Nmax{i}{1} + \frac{3}{4}\telta_i \Nmax{i}{r+1} - \frac{3}{4}\telta_i\Nmax{i}{r}\\
    &\le \frac{3(r_i(T)-r-2)}{8}\telta_i\Nmax{i}{1} + \frac{3}{2}\telta_i \Nmax{i}{r} - \frac{3}{4}\telta_i \Nmax{i}{r} \\
    &= \frac{3(r_i(T)-r-2)}{8}\telta_i\Nmax{i}{1} + \frac{3}{4}\telta_i \Nmax{i}{r}
\end{align*}

If \Cref{alg:error phase} was triggered and $\Nmax{i}{r} \ne \Nmax{i}{1}$ it means that $\Nmax{i}{r+1} = \frac{1}{2}\Nmax{i}{r}$. We have:
\begin{align*}
    \bbE_i^r\squary{\abs{S_i^{r}}\sucb{\tilde{S}_i} - L_i^r} &\le   \frac{3}{8}\telta_i \Nmax{i}{r}
    \tag{\Cref{lem:single phase error}}
    \\
    \bbE_i^r\squary{\sum_{r'=r}^{r_i(T)}\squary{\abs{S_i^{r'}}\sucb{\tilde{S}_i} - L_i(S_i^{r'})}} &= \bbE_i^r\squary{\sum_{r'=r+1}^{r_i(T)}\squary{\abs{S_i^{r'}}\sucb{\tilde{S}_i} - L_i(S_i^{r'})}} + \bbE_i^r\squary{\abs{S_i^{r}}\sucb{\tilde{S}_i} - L_i^r} \\
    &\le \frac{3(r_i(T)-r-3)}{8}\telta_i\Nmax{i}{1} + \frac{3}{4}\telta_i \Nmax{i}{r+1} + \frac{3}{8}\telta_i\Nmax{i}{r} \\
    &\le \frac{3(r_i(T)-r-2)}{8}\telta_i\Nmax{i}{1} + \frac{3}{8}\telta_i \Nmax{i}{r} + \frac{3}{8}\telta_i\Nmax{i}{r} \\
    &= \frac{3(r_i(T)-r-2)}{8}\telta_i\Nmax{i}{1} + \frac{3}{4}\telta_i \Nmax{i}{r}
\end{align*}

If $\Nmax{i}{r} = \Nmax{i}{1}$ it means that $\Nmax{i}{r+1} = \Nmax{i}{r}$. We have:
\begin{align*}
    \bbE_i^r\squary{\abs{S_i^{r}}\sucb{\tilde{S}_i} - L_i^r} &\le   \frac{3}{8}\telta_i \Nmax{i}{1}
    \tag{\Cref{lem:single phase error}}
    \\
    \bbE_i^r\squary{\sum_{r'=r}^{r_i(T)}\squary{\abs{S_i^{r'}}\sucb{\tilde{S}_i} - L_i(S_i^{r'})}} &= \bbE_i^r\squary{\sum_{r'=r+1}^{r_i(T)}\squary{\abs{S_i^{r'}}\sucb{\tilde{S}_i} - L_i(S_i^{r'})}} + \bbE_i^r\squary{\abs{S_i^{r}}\sucb{\tilde{S}_i} - L_i^r} \\
    &\le \frac{3(r_i(T)-r-3)}{8}\telta_i\Nmax{i}{1} + \frac{3}{4}\telta_i \Nmax{i}{r+1} + \frac{3}{8}\telta_i\Nmax{i}{1} \\
    &=\frac{3(r_i(T)-r-2)}{8}\telta_i\Nmax{i}{1} + \frac{3}{4}\telta_i \Nmax{i}{r}
\end{align*}
\end{proof}

\begin{lemma} \label{lem:inactive loss bound}
For every arm $i$:
\begin{align*}
    \E{\sum_{t=\tau_i}^{\bar{T}}\squary{\sucb{\bar{S}} - l_i(t)}} \le 594\sqrt{KT\logp{T}}
\end{align*}
\end{lemma}
\begin{proof}
\Cref{lem:phases losses} with $r=1$ gives:
\begin{align*} 
    \bbE_i^1\squary{\sum_{r'=1}^{r_i(T)}\squary{\abs{S_i^{r'}}\sucb{\tilde{S}_i} - L_i(S_i^{r'})}} \le \frac{3(r_i(T) - 3)}{8}\telta_i\Nmax{i}{1} + \frac{3}{4}\telta_i\Nmax{i}{1} \le \frac{3r_i(T)}{8}\telta_i\Nmax{i}{1}
\end{align*}

From \Cref{lem:max phases}:
\begin{align} \label{eq:total phases losses}
    \bbE_i^1\squary{\sum_{r'=1}^{r_i(T)}\squary{\abs{S_i^{r'}}\sucb{\tilde{S}_i} - L_i(S_i^{r'})}} \le \frac{21\logp{T}}{8}\telta_i\Nmax{i}{1}
\end{align}

If $\frac{n_i\roundy{\tilde{S}_i}}{T} \ge \frac{1}{K}$:
\begin{align*}
    p_i^1\telta_i &\ge \frac{n_i\roundy{\tilde{S}_i}}{2T}8\sqrt{\frac{2\logp{T}}{n_i\roundy{\tilde{S}_i}}}\\
    &= \sqrt{\frac{32\logp{T}}{T}}\sqrt{\frac{n_i\roundy{\tilde{S}_i}}{T}}\\
    &\ge \sqrt{\frac{32\logp{T}}{KT}}
\end{align*}

If $\frac{n_i\roundy{\tilde{S}_i}}{T} \le \frac{1}{K}$:
\begin{align*}
    p_i^1\telta_i &\ge  \frac{8}{2K}\sqrt{\frac{2\logp{T}}{n_i\roundy{\tilde{S}_i}}} 
    \\
    &= \sqrt{\frac{32\logp{T}}{K}}\sqrt{\frac{1}{Kn_i\roundy{\tilde{S}_i}}}\\
    &\ge \sqrt{\frac{32\logp{T}}{KT}}
\end{align*}

In any case:
\begin{align} \label{eq:p telta lb}
    p_i^1\telta_i \ge \sqrt{\frac{32\logp{T}}{KT}}
\end{align}

Since for every $i$, $\tilde{S}_i$ is a sub-series of $\bar{S}$, we have $\sucb{\tilde{S}_i} \ge \sucb{\bar{S}}$. From \Cref{eq:total phases losses,eq:p telta lb}:
\begin{align*}
    \bbE_i^1\squary{\sum_{t=\tau_i}^{\bar{T}}\squary{\sucb{\bar{S}} - l_i(t)}} &\le \bbE_i^1\squary{\sum_{t=\tau_i}^{\bar{T}}\squary{\sucb{\tilde{S}_i} - l_i(t)}}\\
    &\le \frac{21\logp{T}}{8}\telta_i\Nmax{i}{1} \\
    &= \frac{3360\logp{T}}{p_i^1\telta_i} \\
    &\le 594\sqrt{KT\logp{T}}
\end{align*}
\end{proof}

\begin{theorem} \label{thm:adv regret}
\begin{align*}
    R_{adv} \le O\roundy{\sqrt{KT\logp{T}} + \logp{K}\sigma_{max} + R_{\texttt{ALG}}}
\end{align*}
\end{theorem}
\begin{proof}
From \Cref{alg:second switch}:
\begin{align}
    \sum_{s\in \bar{S}}[l_{a_s}(s) - \text{ucb}^*(\bar{S})] \le 272\sqrt{KT\logp{T}} +10\sigma_{max}\logp{K} \notag \\
    \sum_{t=1}^{\bar{T}}[l_{a_t}(t)] - \bar{T}\text{ucb}^*(\bar{S}) \le 272\sqrt{KT\logp{T}} + 11\sigma_{max}\logp{K} \label{eq:regret ucb}
\end{align}

From Wald's equation and \Cref{alg:first switch}, for every arm $i$: 
\begin{align*}
     \E{\sum_{t\in\tilde{{S}_i}}l_i(t)} &= \E{\abs{\tilde{{S}_i}} \omu{i}{\tilde{S}_i}} \\
     &\ge \E{\abs{\tilde{{S}_i}} \olcb{i}{\tilde{S}_i} - \owidth{\tilde{S}_i}} \\
     &\ge \E{\abs{\tilde{{S}_i}} \roundy{\oucb{i}{\tilde{S}_i} - 3\owidth{\tilde{S}_i}}} \\
     &\ge \E{\abs{\tilde{{S}_i}} \roundy{\sucb{\tilde{S}_i} - 3\sqrt{\frac{2K\logp{T}}{\abs{\tilde{{S}_i}}}}}} \\
    &\ge \E{\abs{\tilde{{S}_i}} \sucb{\bar{S}} - 3\sqrt{2KT\logp{T}}} \\ 
\end{align*}
Adding the missing pulls we get:
\begin{align} \label{eq:active loss bound}
    \E{\abs{\tilde{{S}_i}} \sucb{\bar{S}} -\sum_{t=1}^{\tau_i-1}l_i(t)} \le 3\sqrt{2KT\logp{T}} + \sigma(\tau_i-1)
\end{align}

From \Cref{eq:active loss bound,eq:regret ucb,lem:inactive loss bound}, for every arm $i$:
\begin{align*}
    \E{{\sum_{t=1}^{\bar{T}}\squary{l_{a_t}(t) - l_i(t)}}} &=\E{\sum_{t=1}^{\bar{T}}[l_{a_t}(t)] - \bar{T}\text{ucb}^*(\bar{S})} \\
    &\quad+ \E{(\tau_i-1) \sucb{\bar{S}} - \sum_{t=1}^{\tau_i-1}l_i(t)} 
    \\&\quad+ \E{\sum_{t=\tau_i}^{\bar{T}}\squary{\sucb{\bar{S}} - l_i(t)}}\\
    &\le 869\sqrt{KT\logp{T}} + 12\sigma_{max}\logp{K}
\end{align*}
Since after $\bar{T}$ the algorithm switches to \texttt{ALG}, we have:
\begin{align*}
    \E{{\sum_{t=\bar{T}+1}^{T}\squary{l_{a_t}(t) - l_i(t)}}} \le R_{\texttt{ALG}}
\end{align*}
Which concludes the proof.
\end{proof}

\begin{lemma}[Restatement of \Cref{lem:sigma bound}]
\begin{align*}
    \sigma_{max} \le O\roundy{\min_{S\in [T]} \curly{\abs{S} + \sqrt{D_{\bar{S}}}}}
\end{align*}
\end{lemma}
\begin{proof}
Let $S^*$ be the set that minimizes $\abs{S} + \sqrt{D_{\bar{S}}}$. If $\abs{S^*} \ge \frac{1}{2}\sigma_{max}$ it concludes the proof. Continuing with the case that $\abs{S^*} < \frac{1}{2}\sigma_{max}$.

Let $t$ be the step such that $\sigma(t) = \sigma_{max}$. Since $\abs{S^*} < \frac{1}{2}\sigma_{max}$, after skipping there are at least $\frac{1}{2}\sigma_{max}$ non-skipped missing steps at time $t$. Let $s_1,...,s_{\frac{1}{2}\sigma(t)} \in \bar{S^*}$ be the series of those $\frac{1}{2}\sigma(t)$ missing steps, ordered in descending order of when they were pulled. Namely, $s_1$ is the most recent pull in the series and $s_{\frac{1}{2}\sigma(t)}$ is the oldest pull. 

Since there are at least $i-1$ missing pulls that were pulled after $s_i$, we have $t-s_i \ge i$. Additionally, since $s_i$ is missing, we have $s_i + d_{s_i} > t$. Combining both we have $d_{s_i} > i$. Thus:
\begin{align*}
    D_{\bar{S^*}} \ge \sum_{i=1}^{\frac{1}{2}\sigma(t)}d_{s_i} > \sum_{i=1}^{\frac{1}{2}\sigma(t)} i \ge \frac{\roundy{\frac{1}{2}\sigma(t)}^2}{2} = \frac{1}{8}\sigma(t)^2
\end{align*}
\end{proof}

\section{Removing the $\logp K$ factor}
\label{appendix:log K}
We show that a simple modification of the algorithm can eliminate the $\logp K$ factor from the additive delay term in both the adversarial and stochastic settings (\Cref{thm: sto regret,thm:adv regret}). To avoid adding complexity to the already intricate algorithm, we present this modification separately as an optional, opt-in feature.

\begin{algorithm}[H]
\caption{Delayed SAPO Algorithm with reduced $\logp K$}
\label{alg:main logK}
\begin{algorithmic}[1]
\Require Number of arms $K$, number of rounds $T \geq K$, Algorithm \texttt{ALG}.
\State Initialize active arms $\mathcal{A} = \{1, \dots, K\}$, $S = \series{}$, $h=1$, $G = \emptyset$
\For{$t = 1, 2, \dots, T$}
    \For{$s\in B \setminus S$}
        \State $S = S + \series{s}$
        \If{not $\texttt{BSC}(S)$ (Procedure~\ref{alg:bsc})}
                \State Switch to \texttt{ALG}.
        \EndIf
        \State $U(t) = \{i \in \mathcal{A}: \hat{\mu}_{i}(S) - 9\width{i}{S} > \sucb{S}\}$ \Comment{Ghosting}
        \State $G = G \cup U$
        \For{$i \in U$} \Comment{Initialization for phases variables}
            \State Set $p_i^1 = \frac{1}{2K} + \frac{n_i(S)}{2T}$, $\tilde{S}_i = S$, $\tmu_i = \hat{\mu}_i(S)$, $\telta_i = 8\width{i}{S}$, $\Nmax{i}{1} :=  1280 / (p_i^1 \tilde{\Delta}_i^2)$, $E_i = 0$, $r_i=1$, $S_i^1 = \series{}$, $C_i^{p_i^1\cdot2^{-j}}=\emptyset\;\forall j\in[\logp{T}]$ 
        \EndFor
        \If{$\min_i \width{i}{S} \le 2^{-h}$}\Comment{Elimination point}
            \For{$i\in G$}
                \State $\tau_i = t$, $S^g_i = S \setminus \tilde{S}_i$
            \EndFor
            \State $\A = \A \setminus G$, $G = \emptyset$, $h = h+1$
            
        \EndIf
    \EndFor
    
    \For{$i \in \bar{\A}$}
        \State $p_i(t), err = \texttt{EAP}(i)$ (Procedure~\ref{alg:eap})
        \If{$err$}
            \State Switch to \texttt{ALG}. 
        \EndIf
    \EndFor
    \State $\forall i\in \A\quad p_i(t) = \left(1 - \sum_{j \in \bar{\A}(t)} p_j(t)\right)/|\mathcal{A}(t)|$
    \State \text{Observe feedback and update variables}
\EndFor
\end{algorithmic}
\end{algorithm}

\begin{algorithm} 
\floatname{algorithm}{Procedure}
\caption{Basic Stochastic Checks (BSC) Subroutine with reduced $\logp K$}
\label{alg:bsc logK}
\begin{algorithmic}[1]
\Require Series of processed pulls $S$
\If{\(\exists i \in \mathcal{A} : \omu{i}{S} \not\in [\olcb{i}{S} - \owidth{S}, \oucb{i}{S} + \owidth{S}]\)}\label{alg:first switch}
                \State \Return False 
        \EndIf
        \If{
            $\sum_{s'\in S}\roundy{l_{a_{s'}}(s') - \text{ucb}^*(S)}> C\roundy{\sqrt{KT\logp{T}} + \sigma_{max}(t)}$
        }\label{alg:second switch}
            \State \Return False 
        \EndIf
    \Return True
\end{algorithmic}
\end{algorithm}

The key change involves introducing elimination points, where the $h$th elimination point is when the confidence width of at least one arm falls below $2^{-h}$. When an arm is eliminated under the current algorithm, it enters a ghost period—a phase during which it remains practically active (receiving the same pull probability as active arms) and is then formally eliminated at the next elimination point. Additionally, we modify the threshold in \textit{BSC}'s \Cref{alg:second switch} by removing the $\logp K$ term from its additive component.

We first show that the leading term in the stochastic regret remains asymptotically unchanged, since the number of pulls during the ghost period is asymptotically smaller than the number of pulls during the active period (\Cref{lem:logK sto reg}). We then prove a variant of \Cref{lem:delta avg} without the $\logp K$ factor (\Cref{lem:delta avg no logK}), which is the original source of this term in the regret bound. With these changes, the updated version of \textit{BSC}'s \Cref{alg:second switch} (with an appropriate choice of the constant $C$) still doesn't triggers a switch in the stochastic settings.

The $\logp K$ factor is also removed from the adversarial regret, without affecting the asymptotic behavior. The main contribution to adversarial regret comes from the check in \textit{BSC}'s \Cref{alg:second switch}, where we removed the $\logp K$ term. We also need to verify that the relation between the losses and $\text{ucb}^*$ given in \Cref{eq:active loss bound} still holds. This is indeed the case, since the check in \textit{BSC}'s \Cref{alg:first switch} continues to be valid during the ghost period (as $i \in \A$ still holds).

\begin{lemma} \label{lem:logK sto reg}
Assume $G_{sto}$, we have:
\begin{align*}
    n(S_i^g) = O\roundy{n_i(\tilde{S}_i)}
\end{align*}
\end{lemma}
\begin{proof}
Assume $i$ is eliminated in the $h$th elimination point and denote $S_h$ to be $S$ at that time. Let $i_h$ be the arm whose \textit{width} crossed $2^{-h}$ in the $h$th elimination point. From the definition of \textit{width}, we still have that its \textit{width} is greater then $\frac{1}{2}2^{-h}$. Thus:
\begin{align*}
    \width{i}{S_h} &\ge \width{i_h}{S_h} \ge \frac{1}{2}2^{-h}  \\
    \width{i}{S_{h-1}} &\le 10\width{i_{h-1}}{S_{h-1}} \le 10\cdot 2^{-h+1} = 20\cdot 2^{-h} \tag{\Cref{lem:delta bound}} \\
    40\sqrt{\frac{2\logp T}{n_i(S_h)}} &\ge \sqrt{\frac{2\logp T}{n_i(S_{h-1})}} \\
    1600n_i(\tilde{S}_i) &\ge 1600n_i(S_{h-1}) \ge n_i(S_h) \ge n_i(S_i^g)
\end{align*}

\end{proof}

\begin{lemma} \label{lem:elim point delta bound}
Assume $G_{sto}$, for every $i\in[d]$ that is eliminated at the $h$th elimination point we have:
\begin{align*}
    8\cdot 2^{-h} \le \Delta_i \le 320\cdot 2^{-h}
\end{align*}
\end{lemma}
\begin{proof}
Let $S_h$ be $S$ at the time of $h$th elimination point. From \Cref{lem:delta bound}:
\begin{align*}
    \width{i}{S_h} &\ge 2^{-h} \\
    \width{i}{S_{h-1}} &\le 10\cdot 2^{-h+1} = 20\cdot 2^{-h}
\end{align*}
Since $\abs{S_{h-1}} \le \abs{\tilde{S}_i} \le \abs{S_h}$:
\begin{align*}
    2^{-h} \le \width{i}{\tilde{S}_i} \le 20\cdot 2^{-h}
\end{align*}
Again from \Cref{lem:delta bound}:
\begin{align*}
    8\cdot 2^{-h} \le \Delta_i \le 320\cdot 2^{-h}
\end{align*}
\end{proof}

\begin{lemma} \label{lem:delta avg no logK}
Assume $G_{sto}$, we have:
\begin{align*}
    \sum_i p_i^{max} \Delta_i = O\roundy{\Delta_{avg}}
\end{align*}
\end{lemma}
\begin{proof}
In the same way as \Cref{lem:delta avg}, we denote $\A$' to be the state of $\A$ such that $\abs{A} = \frac{K}{2}$. In the same way we have that $\forall i\in\A'\quad \Delta_i \le 40\Delta_{avg}$ and:
\begin{align*}
    \sum_i p_i^{max}\Delta_i \le \frac{\Delta_{avg}}{2} + \sum_{i\in\A'}p_i^{max}\Delta_i
\end{align*}
Fix elimination point $h$, denote $I_h$ to be the set of arms eliminated at that point. By definition, we have for every $i\in I_h$ that $p_i^{max} \le \frac{1}{\abs{I_h}}$. From \Cref{lem:elim point delta bound}:
\begin{align*}
    \sum_{i\in I_h}p_i^{max}\Delta_i \le \sum_{i\in I_h}\frac{1}{\abs{I_h}}320\cdot 2^{-h} = 320\cdot 2^{-h} 
\end{align*}

Let $h_1$ be the first elimination point in which arms from $\A$' are eliminated. Again from \Cref{lem:elim point delta bound}, we have for some $i\in\A'$:
\begin{align*}
    8\cdot 2^{-h_1} &\le \Delta_i \le 40\Delta_{avg} \\
    2^{-h_1} &\le 5\Delta_{avg}
\end{align*}

This concludes to:
\begin{align*}
    \sum_{i\in\A'}p_i^{max}\Delta_i = \sum_{h=h_1}^\infty\sum_{i\in I_h}p_i^{max}\Delta_i \le \sum_{h=h_1}^\infty320\cdot 2^{-h}  = 640\cdot 2^{-h_1} \le 3200\Delta_{avg}
\end{align*}

\end{proof}
\end{document}